\documentclass[11pt,letterpaper]{article}

\usepackage[margin=1in]{geometry}
\usepackage[T1]{fontenc}
\usepackage{lmodern}
\usepackage{microtype}
\usepackage{amsmath,amssymb,amsthm}
\usepackage{chngcntr}
\allowdisplaybreaks[1]

\usepackage{graphicx}
\usepackage[authoryear,round]{natbib}
\bibpunct{(}{)}{;}{a}{,}{,}
\usepackage{hyperref}

\hypersetup{colorlinks=true,linkcolor=blue,citecolor=blue,urlcolor=blue}

\theoremstyle{plain}
\newtheorem{theorem}{Theorem}
\newtheorem{lemma}[theorem]{Lemma}
\newtheorem{proposition}[theorem]{Proposition}
\theoremstyle{remark}
\newtheorem{remark}[theorem]{Remark}

\newcommand{\E}{\mathbb E}
\newcommand{\R}{\mathbb R}
\newcommand{\cM}{\mathcal M}
\newcommand{\cP}{\mathcal P}
\newcommand{\vol}{\operatorname{vol}}
\newcommand{\tr}{\operatorname{tr}}

\newcommand{\TV}{\operatorname{TV}}
\newcommand{\KL}{\operatorname{KL}}
\newcommand{\supp}{\operatorname{supp}}

\newcommand{\dd}{\mathrm d}
\newcommand{\FLIPD}{\mathrm{FLIPD}}

\title{Minimax Rate for Diffusion-Guided Local Intrinsic Dimension Estimation based on Fokker-Planck Equation}

\author{%
  Jaehee Seo \quad Wontae Jeong \quad Jisu Kim\thanks{Corresponding author.}\\[0.5em]
  \small Department of Statistics\\
  \small Seoul National University\\[0.4em]
  \small \texttt{\{seojaehee02, wtjeong,jkim82133\}@snu.ac.kr}
}
\date{}

\hypersetup{
  pdfauthor={Jaehee Seo, Wontae Jeong, Jisu Kim},
  pdftitle={Minimax Rate for Diffusion-Guided Local Intrinsic Dimension Estimation based on Fokker-Planck Equation}
}

\begin{document}

\maketitle

\begin{abstract}
Diffusion-guided methods have recently emerged as effective tools for
probing the intrinsic geometry of high-dimensional data. Understanding
their accuracy requires quantifying both (i) finite-scale approximation
error relative to the true local intrinsic dimension (LID) and (ii)
estimation error from finite samples. We address these two questions for
the FLIPD field \citep{kamkari2024geometric}, defined at the
population level through the logarithmic scale derivative of a
Gaussian-smoothed density and admitting a Fokker--Planck representation. For $\alpha$-smooth densities on a regular $d$-dimensional manifold, we establish a uniform
$O(\sigma^2)$ approximation bound relative to the manifold dimension $d$.
Under expected integrated squared loss, we prove a minimax lower bound
of order $(n\sigma^d)^{-1}$ for estimating the FLIPD field from
$n$ observations when $n^{-1/(2\alpha+d)}\le\sigma\le\sigma_0$.
A plug-in FLIPD estimator based on the empirical
denoising score matching (DSM) solution attains a matching upper bound.
Thus the minimax rate is of order $(n\sigma^d)^{-1}$ and it becomes $n^{-2\alpha/(2\alpha+d)}$ at its smallest scale.
\end{abstract}

\par

\section{Introduction}\label{sec:introduction}

Understanding the intrinsic geometry of high-dimensional data has become
increasingly important in modern statistics and machine learning. The
manifold hypothesis suggests that complex data in an ambient space
$\mathbb R^D$ often concentrate near structures whose effective dimension is
much smaller than the ambient dimension $D$ \citep{narayanan2010sample,fefferman2016testing}.
The \emph{local intrinsic dimension} (LID) characterizes the dimension of
the data support near a query point. Its estimators resolve local geometry
at a neighborhood or noise scale, yielding scale-dependent approximations
to the underlying LID. It has become a
useful tool in representation learning \citep{ansuni2019intrinsic}, and has
found applications in out-of-distribution detection
\citep{wang2021dimensionality}, adversarial robustness
\citep{ma2018characterizing}, and the analysis of generative models
\citep{stanczuk2024diffusion}.
Classical LID estimators primarily rely on local statistics, including
nearest-neighbor distances \citep{facco2017estimating},
likelihood-based estimators \citep{levina2005maximum}, and local PCA
approaches \citep{fukunaga1971eigenvalue}.

Recent advances in \emph{diffusion models} (DMs) offer a multiscale
perspective on local intrinsic dimension through the geometry of
Gaussian-smoothed distributions. Score-based DMs learn score functions of
noise-perturbed distributions, with the noise level controlling the scale
at which geometric structure is resolved: large noise levels emphasize
coarse structure, whereas small noise levels reveal the local geometry of
the data support. Diffusion geometry provides a complementary approach to geometric
inference by recovering metric, tangent-space, and curvature information
from diffusion generators estimated from data
\citep{jones2024diffusion,jones2026manifold}.
Our analysis studies the FLIPD field derived from
Gaussian-smoothed densities, whose evolution is governed by the known
ambient heat operator. The
Fokker--Planck equation describes the time evolution of the probability
density of a diffusion process. In the drift-free Brownian case, it
reduces to the heat equation
$\partial_t\rho_t=\tfrac12\Delta_x\rho_t$, where $\rho_t$ is the density
obtained by convolving the data law with $N(0,tI_D)$. Letting $t=\sigma^2$ and
$p_\sigma=\rho_{\sigma^2}$ gives its noise-scale form
\begin{equation}
\label{eq:gaussian-scale-fokker-planck}
 \partial_\sigma p_\sigma(x)
 =\sigma\Delta_x p_\sigma(x).
\end{equation}
Here $\sigma$ denotes the noise standard deviation, and $\Delta_x$ is
the ambient Laplacian on $\R^D$. Writing
$s_\sigma:=\nabla_x\log p_\sigma$ for the score, we obtain
\begin{equation}
\label{eq:scale-score-identity}
 \sigma\partial_\sigma\log p_\sigma(x)
 =\sigma^2\left\{
   \operatorname{div}s_\sigma(x)+\|s_\sigma(x)\|^2
 \right\}.
\end{equation}
Thus the score and its divergence determine the logarithmic scale
derivative at a fixed location $x$
\citep{stanczuk2024diffusion,leung2025convolutions}. This identity follows
from Gaussian smoothing alone; relating it to intrinsic dimension
additionally requires small-noise analysis of the support geometry.
Near a manifold, local volume growth provides this link, motivating
dimension estimation through logarithmic scale derivatives. This
perspective underlies a growing literature on diffusion-guided LID estimation
\citep{tempczyk2022lidl,kamkari2024geometric,osada2026local}.

The statistical theory of diffusion-guided LID estimation is still
developing. \citet{leung2025convolutions} established the small-noise limit
of the FLIPD field on smooth manifolds. Building on this result,
we study two components of accuracy at finite noise scales:
\begin{enumerate}
    \item At a prescribed diffusion scale $\sigma>0$, the FLIPD field
    has a scale-dependent approximation error relative to the true LID.
    \item Estimating the FLIPD field from $n$ observations is subject to
    sampling error. Learned score models may introduce further
    approximation and optimization errors.
\end{enumerate}
This paper aims to quantify both the approximation error of the FLIPD field
relative to the true LID and the statistical difficulty of estimating it:
\begin{center}
\emph{How far is the FLIPD field from the true LID,\\
and how difficult is it to estimate the FLIPD field from $n$ observations?}
\end{center}
Under a regular manifold model, we address the first question through a
uniform finite-scale bias bound (Theorem~\ref{thm:small-noise-limits})
and the second through matching minimax bounds
(Theorems~\ref{thm:minimax-lower} and~\ref{thm:minimax-upper}).
The first result quantifies how rapidly the FLIPD field approaches
the manifold dimension $d$ as $\sigma\to0$; the second characterizes its
estimation difficulty at a prescribed positive scale.

\subsection{Contributions}\label{sec:contribution}

We study the FLIPD field at a prescribed noise scale and analyze both its
deterministic bias relative to the LID and the finite-sample statistical
limits of its estimation. Our contributions are summarized as follows:

\begin{itemize}

    \item \textbf{FLIPD field as a statistical target.}
    Building on the FLIPD formulation \citep{kamkari2024geometric},
    we formulate estimation of the FLIPD field at a prescribed positive
    noise level as a statistical problem.

    \item \textbf{Uniform finite-scale approximation.}
    Theorem~\ref{thm:small-noise-limits} gives, uniformly over the model class,
    \[
    T_\sigma(x;f)=d+O(\sigma^2).
    \]
    Local tangent-plane approximation and Gaussian symmetry cancel the
    first-order terms, leaving curvature and density variation to
    contribute at second order.

    \item \textbf{Finite-sample minimax lower bound.}
    For $n^{-1/(2\alpha+d)}\le\sigma\le\sigma_0$,
    Theorem~\ref{thm:minimax-lower} establishes
    \[
        \mathfrak R_{n,\sigma}\gtrsim(n\sigma^d)^{-1}.
    \]
    Here, the factor $n\sigma^d$ reflects the effective number of observations
    available in an intrinsic $\sigma$-neighborhood. At the smallest scale
    covered by the Assouad construction,
    $\sigma\asymp n^{-1/(2\alpha+d)}$, the lower bound becomes
    $n^{-2\alpha/(2\alpha+d)}$.

    \item \textbf{Empirical DSM and matching upper bound.}
    We show that the exact noise-integrated empirical DSM solution yields
    a raw FLIPD estimator expressible as a Gaussian-weighted moment
    ratio (Proposition~\ref{prop:empirical-dsm}). After clipping,
    Theorem~\ref{thm:minimax-upper} gives, throughout the same scale range,
    \[
        \sup_{P_f\in\cP_\alpha(\cM,d)}
        \mathcal R_{n,\sigma}(\widehat T_\sigma,f)
        \lesssim(n\sigma^d)^{-1}.
    \]
    Together, the bounds establish
    $\mathfrak R_{n,\sigma}\asymp(n\sigma^d)^{-1}$.
\end{itemize}

\section{Main Results}
\label{sec:model-main}

We now define the FLIPD field and state its finite-scale bias bound
and matching minimax bounds. Notation and geometric preliminaries are
collected in Appendix~\ref{app:preliminaries}; the local geometric facts
used in the proofs appear in Appendix~\ref{app:geometry}. The finite-scale bias,
minimax lower bound, and empirical DSM upper bound are proved in
Appendices~\ref{app:proof-bias}, \ref{app:minimax-lower}, and
\ref{app:minimax-upper}, respectively.

\subsection{Statistical Models}\label{sec:statistical-models}
Fix integers $1\le d<D$, density smoothness $\alpha>2$, manifold smoothness $\beta\ge\lceil\alpha\rceil+1$, reach bound $\tau>0$, and norm bound $L_\cM\ge1$. Let $\cM\subset\R^D$ belong to the regular embedded manifold class $\mathcal C^\beta_{d, D, \tau, L_\cM}$, defined in Section~\ref{sec:reach-manifold}. Let $f_0:=\vol_\cM(\cM)^{-1}$ and fix constants $0<c_-<f_0<c_+<\infty$ and $L_f>\|f_0\|_{\mathcal H^\alpha(\cM)}$.  We distinguish the density class from the induced class of probability measures by setting
\begin{align*}
 \mathcal F_\alpha(\cM)
 &:=
 \left\{
 f\in\mathcal H^\alpha(\cM):
 \int_\cM f\,\dd\vol_\cM=1,\quad
 c_-\le f\le c_+,\quad
 \|f\|_{\mathcal H^\alpha(\cM)}\le L_f
 \right\},\\
 \cP_\alpha(\cM,d)
 &:=
 \left\{P_f:P_f(\dd x)=f(x)\,\dd\vol_\cM(x),\ 
 f\in\mathcal F_\alpha(\cM)\right\}.
\end{align*}
\noindent
Here $\mathcal H^\alpha(\cM)$ is defined using the tangent--normal charts in
Section~\ref{sec:holder-manifolds}. Constants in the bounds are uniform in
$n$, the scale $\sigma$, the density $f$, and the query point. They may
depend on $D,d,\alpha,\beta,L_\cM,\tau$ and $c_-,c_+,L_f$.
Constants in the lower-bound arguments may additionally depend on the
fixed interior margin
\[
 \Delta_0:=\min\left\{
 f_0-c_-,\ c_+-f_0,\
 L_f-\|f_0\|_{\mathcal H^\alpha(\cM)}
 \right\}>0.
\]
For $\sigma>0$, we define
\begin{align}
 K_{\sigma,x}(u)
 :=
 \exp\!\left(-\frac{\|u-x\|^2}{2\sigma^2}\right),
 \qquad &
 Z_\sigma(f,x):=
 \int_\cM K_{\sigma,x}(u)f(u)\,d\!\vol_\cM(u),
 \label{eq:kernel-mass}\\[0.4cm]
 p_\sigma(x;f)
 :=
 (2\pi)^{-D/2}\sigma^{-D}Z_\sigma(f,x),
 \qquad &
 s_\sigma(x;f)
 :=
 \nabla_x\log p_\sigma(x;f).
 \label{eq:smoothed-density}
\end{align}

\bigskip

\subsection{FLIPD Field}
\label{sec:LID-field}\label{sec:FLIPD-field}

For fixed $f$ and $\sigma>0$, the FLIPD field is the population function
\[
    x\longmapsto T_\sigma(x;f),
\]
defined from the Gaussian-smoothed density. Throughout, LID refers to the
underlying local intrinsic dimension, equal to $d$ under the present
manifold model. A sample-based estimate of the FLIPD field is called a
FLIPD estimator. The geometric principle is the small-noise scaling of the Gaussian convolution.  At a point $x\in\cM$, a Gaussian kernel of radius $\sigma$ sees a neighborhood of $\cM$ whose intrinsic volume is of order $\sigma^d$, whereas the ambient Gaussian normalizing factor is of order $\sigma^{-D}$.  More precisely, under the regularity conditions above, we have
\begin{equation}\label{eq:log_score}
    \sigma\,\partial_\sigma\log p_\sigma(x;f)
    =d-D+O(\sigma^2),
\end{equation}
so adding the ambient dimension $D$ recovers $d$ in the zero-noise limit.

\medskip
Following the formulation of \cite{kamkari2024geometric}, the FLIPD field
is defined, for each $x\in\cM$, by
\begin{align}
    T_\sigma(x;f)
    :=
    D+\sigma\,\partial_\sigma\log p_\sigma(x;f)
    =
    \sigma\,\partial_\sigma\log Z_\sigma(f,x).
    \label{eq:flipd-target}
\end{align}
The equality in \eqref{eq:flipd-target} follows from
\[
    \log p_\sigma(x;f)
    =
    \log Z_\sigma(f,x)
    -D\log\sigma
    -\frac{D}{2}\log(2\pi).
\]
In particular, if $p_\sigma(x;f)\propto\sigma^{d-D}$ holds locally, then
\[
    T_\sigma(x;f)=D+(d-D)=d.
\]

Combining \eqref{eq:flipd-target} with the score identity
\eqref{eq:scale-score-identity} gives
\begin{equation}
\label{eq:flipd-score}
    T_\sigma(x;f)
    =
    D+\sigma^2
    \left\{
        \operatorname{div}s_\sigma(x;f)
        +\|s_\sigma(x;f)\|^2
    \right\}.
\end{equation}
Here the divergence is taken in the ambient coordinates,
$\operatorname{div}s_\sigma=\tr\{\nabla_xs_\sigma\}$.

\subsection{Finite-Scale Bias}\label{sec:bias}

We first quantify the deterministic finite-scale approximation error of
$T_\sigma$ relative to $d$. Gaussian convolution preserves the leading
volume scaling $\sigma^d$, with corrections from density variation and
local manifold geometry of order $\sigma^2$.

\begin{theorem}[Finite-scale bias]
\label{thm:small-noise-limits}
There exists $\sigma_0>0$ such that, for every $0<\sigma\le\sigma_0$,
\begin{equation}
 \sup_{f\in\mathcal F_\alpha(\cM)}\sup_{x\in\cM}
 |T_\sigma(x;f)-d|\lesssim\sigma^2.
 \label{eq:flipd-limit}
\end{equation}
\end{theorem}

\noindent
Theorem~\ref{thm:small-noise-limits} quantifies the deterministic
approximation error of $T_\sigma$ relative to $d$. The minimax results
below characterize estimation of $T_\sigma$ at a prescribed positive scale.

\subsection{Minimax Lower Bound}\label{sec:minimax-lower}

Minimax theory for manifold dimension studies recovery of the dimension
parameter \citep{Kim2019minimax}. Here we study the expected integrated
squared error for the FLIPD field $T_\sigma(\cdot;f)$ under
the model in Section~\ref{sec:statistical-models}.

A FLIPD estimator $\widehat T_\sigma$ is constructed from
$X_1,\ldots,X_n\stackrel{\mathrm{i.i.d.}}\sim P_f$ and evaluated at the
query $x$ and prescribed scale $\sigma$.
Define the expected integrated squared risk by
\begin{equation}
\label{eq:fixed-scale-risk}
 \mathcal R_{n,\sigma}(\widehat T_\sigma,f)
 :=\E\left[\,\int_\cM
 \left|\widehat T_\sigma(x)-T_\sigma(x;f)\right|^2
 f(x)\,d\!\vol_\cM(x)\right],
\end{equation}
and define the minimax risk by
\begin{equation}
\label{eq:minimax-risk}
 \mathfrak R_{n,\sigma}
 :=\,\inf_{\widehat T_{\sigma}}\sup_{P_f\in\cP_\alpha(\cM,d)}
 \mathcal R_{n,\sigma}(\widehat T_\sigma,f),
\end{equation}
where the infimum is taken over all jointly measurable estimators
$\widehat T_\sigma:\cM^n\times\cM\to\R$.
Throughout the remainder of this paper, set $h_n:=n^{-1/(2\alpha+d)}$. 

\begin{theorem}[Minimax lower bound]
\label{thm:minimax-lower}
There are constants $\sigma_0>0$ and $n_0<\infty$ such that, for every $n\ge n_0$ and $h_n\le\sigma\le\sigma_0$,
\begin{equation}
 \mathfrak R_{n,\sigma}
 \gtrsim (n\sigma^d)^{-1}.
 \label{eq:main-flipd-lower}
\end{equation}
\end{theorem}
\noindent
The theorem is stated at the level of the statistical experiment generated
by $X_1,\ldots,X_n$. Hence it applies to any estimator constructed solely
from these observations, including procedures that first fit a score or
divergence model and subsequently construct a FLIPD estimator.

\begin{remark}[Nonparametric rate]
If $\sigma$ is fixed independently of $n$, the lower bound in
\eqref{eq:main-flipd-lower} has $n^{-1}$ dependence. The matching upper
bound in Theorem~\ref{thm:minimax-upper} shows that this dependence is
optimal at every fixed admissible scale. At the smallest scale covered
by our construction,
$\sigma\asymp h_n=n^{-1/(2\alpha+d)}$,
the minimax rate becomes
\[
    n^{-2\alpha/(2\alpha+d)}.
\]
This has the classical $d$-dimensional $\alpha$-H\"older squared-risk
exponent. The scale $h_n$ marks the H\"older admissibility threshold of
the lower-bound construction.
\end{remark}

\subsection{Empirical Denoising Score Matching}
\label{sec:empirical-dsm}

We first construct a raw FLIPD estimator at a prescribed scale
$\sigma$ using empirical denoising score matching (DSM). The score and
its divergence combine into an exact ratio of Gaussian-weighted moments.
The risk analysis in Section~\ref{sec:minimax-upper} proceeds directly
through this moment ratio under the loss \eqref{eq:fixed-scale-risk}.

\paragraph{Noise-integrated empirical DSM.}
For $X_1,\ldots,X_n\stackrel{\mathrm{i.i.d.}}\sim P_f$, let
$\widehat P_n:=n^{-1}\sum_{i=1}^n\delta_{X_i}$ be the empirical law. Its
convolution with $N(0,\sigma^2I_D)$ has density
\begin{equation}
\label{eq:empirical-gaussian-density}
 \widehat Z_\sigma(y):=\frac1n\sum_{i=1}^nK_{\sigma,y}(X_i),
 \qquad
 \widehat p_\sigma(y):=(2\pi)^{-D/2}\sigma^{-D}\widehat Z_\sigma(y),
 \quad y\in\R^D.
\end{equation}
DSM replaces the unknown score by the known score of the corruption
kernel inside a squared-error objective \citep{vincent2011connection}.
For Gaussian corruption $y=X_i+\sigma\varepsilon$, this conditional
score is $-(y-X_i)/\sigma^2=-\varepsilon/\sigma$. Accordingly, consider
\begin{equation}
\label{eq:empirical-dsm-objective}
 \widehat{\mathcal L}_\sigma(s)
 :=\frac1n\sum_{i=1}^n
 \E_{\varepsilon\sim N(0,I_D)}
 \left[\left\|s(X_i+\sigma\varepsilon)
              +\frac{\varepsilon}{\sigma}\right\|^2\right].
\end{equation}
The Gaussian expectation is evaluated exactly with the observations fixed.
We minimize \eqref{eq:empirical-dsm-objective} over the $C^1$-function class
\[
 \mathcal S_\sigma
 :=\left\{s\in C^1(\R^D;\R^D):
       \int_{\R^D}\|s(y)\|^2\widehat p_\sigma(y)\,\dd y<\infty\right\}.
\]
This unrestricted empirical problem isolates sampling error from the
approximation and optimization errors of a parameterized score model.
The representation below yields a closed-form estimator evaluated through
finite Gaussian kernel sums.

\paragraph{Gaussian-weighted empirical moments.}
For each query $x\in\R^D$, define the normalized Gaussian kernel weights
and the corresponding weighted empirical mean and covariance by
\begin{align}
 w_{i,\sigma}(x)
 &:=\frac{K_{\sigma,x}(X_i)}{\sum_{j=1}^nK_{\sigma,x}(X_j)},
 \qquad
 \widehat m_\sigma(x):=\sum_{i=1}^nw_{i,\sigma}(x)X_i,
 \label{eq:empirical-gaussian-weights}\\
 \widehat C_\sigma(x)
 &:=\sum_{i=1}^nw_{i,\sigma}(x)
 \{X_i-\widehat m_\sigma(x)\}
 \{X_i-\widehat m_\sigma(x)\}^{\mathsf T}.
 \label{eq:empirical-gaussian-covariance}
\end{align}
The weights are nonnegative and sum to one. Thus $\widehat m_\sigma(x)$
is a local weighted average, and $\widehat C_\sigma(x)$ is the associated
centered second-moment matrix. Differentiating the Gaussian kernel sum
expresses its score through $\widehat m_\sigma(x)-x$ and its score
Jacobian through $\widehat C_\sigma(x)$. These are Gaussian
score--moment identities \citep{efron2011tweedie,dytso2022conditional};
here they follow directly from finite-sum differentiation. Together
with completion of squares in the DSM objective, they give the following
exact representation.

\begin{proposition}[Empirical DSM representation]
\label{prop:empirical-dsm}
The unique minimizer of
\eqref{eq:empirical-dsm-objective} over $\mathcal S_\sigma$ is
\begin{equation}
\label{eq:empirical-dsm-minimizer}
 \widehat s_\sigma
 =\operatorname*{arg\,min}_{s\in\mathcal S_\sigma}
       \widehat{\mathcal L}_\sigma(s)
 =\nabla\log\widehat p_\sigma.
\end{equation}
For every $x\in\R^D$,
\begin{equation}
\label{eq:empirical-score-covariance}
 \widehat s_\sigma(x)
 =\frac{\widehat m_\sigma(x)-x}{\sigma^2},
 \qquad
 \nabla_x\widehat s_\sigma(x)
 =-\frac{I_D}{\sigma^2}+\frac{\widehat C_\sigma(x)}{\sigma^4}.
\end{equation}
Consequently, the raw FLIPD estimator obtained from this score satisfies
\begin{align}
 \widehat T_\sigma^{\mathrm{raw}}(x)
 &:=D+\sigma^2\left\{
       \operatorname{div}\widehat s_\sigma(x)
       +\|\widehat s_\sigma(x)\|^2\right\}
 \notag\\
 &=\frac{\tr\widehat C_\sigma(x)
             +\|\widehat m_\sigma(x)-x\|^2}{\sigma^2}
 \notag\\
 &=\frac{\sum_{i=1}^n\|X_i-x\|^2K_{\sigma,x}(X_i)}
          {\sigma^2\sum_{i=1}^nK_{\sigma,x}(X_i)}.
 \label{eq:empirical-dsm-flipd}
\end{align}
All spatial derivatives are ambient derivatives in $\R^D$.
\end{proposition}

The raw FLIPD estimator is the normalized second moment about the
query $x$. Its decomposition contains both the covariance trace and the
squared displacement of the weighted mean. The final expression in
\eqref{eq:empirical-dsm-flipd} evaluates the estimator directly through
Gaussian kernel sums.

Proposition~\ref{prop:empirical-dsm} is proved in
Appendix~\ref{app:minimax-upper}.

\subsection{Minimax Upper Bound}
\label{sec:minimax-upper}

We now analyze a clipped FLIPD estimator constructed from the raw
FLIPD estimator in Section~\ref{sec:empirical-dsm} under
the integrated squared loss \eqref{eq:fixed-scale-risk}.
Truncation controls the effect of small empirical kernel masses.
For uniform risk control, we use the bounded range of the FLIPD field:
the population moment representation
\eqref{eq:population-empirical-flipd-ratios} gives $T_\sigma\ge0$, and
Theorem~\ref{thm:small-noise-limits} gives $T_\sigma\le d+1\le D$ for
all sufficiently small $\sigma$, uniformly in $f$ and $x\in\cM$.
Define the clipped FLIPD estimator by
\begin{equation}
\label{eq:clipped-empirical-flipd}
 \widehat T_\sigma(x)
 :=0\lor\widehat T_\sigma^{\mathrm{raw}}(x)\land(D+1).
\end{equation}
On this scale range, projection onto $[0,D+1]$ contracts the pointwise error:
\[
 |\widehat T_\sigma(x)-T_\sigma(x;f)|
 \le |\widehat T_\sigma^{\mathrm{raw}}(x)-T_\sigma(x;f)|.
\]
The estimator is evaluated from the observations, the query point, the
scale $\sigma$, and the ambient dimension $D$.

\begin{theorem}[Matching upper bound]
\label{thm:minimax-upper}
There exists $\sigma_0>0$ such that, for every $n\ge1$ and
$0<\sigma\le\sigma_0$, the estimator
\eqref{eq:clipped-empirical-flipd} satisfies
\begin{equation}
\label{eq:main-flipd-upper}
 \sup_{P_f\in\cP_\alpha(\cM,d)}
 \mathcal R_{n,\sigma}(\widehat T_\sigma,f)
 \lesssim 1\wedge(n\sigma^d)^{-1}.
\end{equation}
Consequently, for all sufficiently large $n$ and
$h_n\le\sigma\le\sigma_0$,
\begin{equation}
\label{eq:matching-flipd-minimax}
 \mathfrak R_{n,\sigma}\asymp(n\sigma^d)^{-1}.
\end{equation}
\end{theorem}

The rate follows from estimating the numerator and denominator at the
same scale as the target. Writing
$N_\sigma=\sigma\partial_\sigma Z_\sigma$ and
$\widehat N_\sigma=\sigma\partial_\sigma\widehat Z_\sigma$, both
$\widehat N_\sigma$ and $\widehat Z_\sigma$ are unbiased for their
population counterparts. Lemma~\ref{lem:localization} gives
$Z_\sigma\asymp\sigma^d$ and mean squared fluctuations of order at most
$\sigma^d/n$ for both empirical moments. The clipped-ratio inequality in
Lemma~\ref{lem:clipped-ratio} converts these moment errors into the asserted
risk bound through the squared normalization $Z_\sigma^2\asymp\sigma^{2d}$.

Theorem~\ref{thm:minimax-upper} is proved in
Appendix~\ref{app:minimax-upper}. The theorem applies to the exact
noise-integrated empirical DSM solution. The excess-loss identity
\eqref{eq:empirical-dsm-pythagorean} measures score error in
$L^2(\widehat p_\sigma)$. Transferring the guarantee to a parameterized
score model requires additional control of approximation and optimization
errors, together with score and divergence errors on $\cM$.

\section{Conclusion}\label{sec:conclusion}

We studied the FLIPD field under a regular manifold model. Our analysis
quantifies both its approximation to the LID and its finite-sample
estimation difficulty. First, the FLIPD field has a uniform
$O(\sigma^2)$ approximation error,
with the second-order terms arising from density variation and manifold
geometry. Second, its minimax expected integrated squared risk is of
order $(n\sigma^d)^{-1}$ over the stated scale range, attained by the
clipped FLIPD estimator. The rate is governed by the intrinsic
local sample size $n\sigma^d$, with the ambient dimension entering
through the model-dependent constants.

Our minimax result concerns estimation of the FLIPD field on a fixed
$d$-dimensional manifold at a prescribed smoothing scale.
Extending the analysis to heterogeneous geometric supports, including
stratified spaces with spatially varying dimension, is an important
direction for future research
\citep{aamari2024theory,martinez2026deep, seo2026nonparametric}.

Another important direction is to connect this population-level theory to
practical estimators based on pretrained DMs by explicitly
accounting for the approximation and estimation errors associated with
learned score and divergence fields. The statistical properties of other
diffusion-guided LID formulations, including LIDL
\citep{tempczyk2022lidl} and LHSD \citep{osada2026local}, also remain to be
understood.

\section*{Acknowledgments}
Jaehee Seo was supported by the Next Generation Scholarship for Basic Studies (Type C) from Seoul National University.

\bibliographystyle{plainnat}
\bibliography{diffusion_LID}

@article{aamari2019estimating,
  author  = {Aamari, Eddie and Kim, Jisu and Chazal, Fr{\'e}d{\'e}ric and Michel, Bertrand and Rinaldo, Alessandro and Wasserman, Larry},
  title   = {Estimating the Reach of a Manifold},
  journal = {Electronic Journal of Statistics},
  volume  = {13},
  number  = {1},
  pages   = {1359--1399},
  year    = {2019}
}

@article{aamari2019nonasymptotic,
  author  = {Aamari, Eddie and Levrard, Cl{\'e}ment},
  title   = {Nonasymptotic Rates for Manifold, Tangent Space and Curvature Estimation},
  journal = {The Annals of Statistics},
  volume  = {47},
  number  = {1},
  pages   = {177--204},
  year    = {2019}
}

@article{federer1959curvature,
  author  = {Federer, Herbert},
  title   = {Curvature Measures},
  journal = {Transactions of the American Mathematical Society},
  volume  = {93},
  number  = {3},
  pages   = {418--491},
  year    = {1959}
}

@inproceedings{kamkari2024geometric,
  author    = {Kamkari, Hamidreza and Ross, Brendan Leigh and Hosseinzadeh, Rasa and Cresswell, Jesse C. and Loaiza-Ganem, Gabriel},
  title     = {A Geometric View of Data Complexity: Efficient Local Intrinsic Dimension Estimation with Diffusion Models},
  booktitle = {Advances in Neural Information Processing Systems},
  volume    = {37},
  pages     = {38307--38354},
  year      = {2024}
}

@article{Kim2019minimax,
  author  = {Kim, Jisu and Rinaldo, Alessandro and Wasserman, Larry},
  title   = {Minimax Rates for Estimating the Dimension of a Manifold},
  journal = {Journal of Computational Geometry},
  volume  = {10},
  number  = {1},
  pages   = {42--95},
  year    = {2019}
}

@article{vincent2011connection,
  author  = {Vincent, Pascal},
  title   = {A Connection Between Score Matching and Denoising Autoencoders},
  journal = {Neural Computation},
  volume  = {23},
  number  = {7},
  pages   = {1661--1674},
  year    = {2011}
}

@article{aamari2024theory,
  author  = {Aamari, Eddie and Berenfeld, Cl{\'e}ment},
  title   = {A Theory of Stratification Learning},
  journal = {arXiv preprint arXiv:2405.20066},
  year    = {2024}
}

@inproceedings{ansuni2019intrinsic,
  author    = {Ansuini, Alessio and Laio, Alessandro and Macke, Jakob H. and Zoccolan, Davide},
  title     = {Intrinsic Dimension of Data Representations in Deep Neural Networks},
  booktitle = {Advances in Neural Information Processing Systems},
  volume    = {32},
  pages     = {6111--6122},
  year      = {2019}
}

@article{facco2017estimating,
  author    = {Facco, Elena and d'Errico, Maria and Rodriguez, Alex and Laio, Alessandro},
  title     = {Estimating the Intrinsic Dimension of Datasets by a Minimal Neighborhood Information},
  journal   = {Scientific Reports},
  volume    = {7},
  number    = {1},
  pages     = {12140},
  publisher = {Nature Publishing Group},
  year      = {2017}
}

@article{fukunaga1971eigenvalue,
  author    = {Fukunaga, Keinosuke and Olsen, David R.},
  title     = {An Algorithm for Finding Intrinsic Dimensionality of Data},
  journal   = {IEEE Transactions on Computers},
  volume    = {C-20},
  number    = {2},
  pages     = {176--183},
  publisher = {IEEE},
  year      = {1971}
}

@article{leung2025convolutions,
  author  = {Leung, Kin Kwan and Hosseinzadeh, Rasa and Loaiza-Ganem, Gabriel},
  title   = {On Convolutions, Intrinsic Dimension, and Diffusion Models},
  journal = {Transactions on Machine Learning Research},
  issn    = {2835-8856},
  year    = {2025}
}

@inproceedings{levina2005maximum,
  author    = {Levina, Elizaveta and Bickel, Peter J.},
  title     = {Maximum Likelihood Estimation of Intrinsic Dimension},
  booktitle = {Advances in Neural Information Processing Systems},
  volume    = {17},
  year      = {2004}
}

@inproceedings{
ma2018characterizing,
title={Characterizing Adversarial Subspaces Using Local Intrinsic Dimensionality},
author={Xingjun Ma and Bo Li and Yisen Wang and Sarah M. Erfani and Sudanthi Wijewickrema and Grant Schoenebeck and Michael E. Houle and Dawn Song and James Bailey},
booktitle={International Conference on Learning Representations},
year={2018}
}

@article{martinez2026deep,
  author  = {Martinez, Randy and Tang, Rong and Lin, Lizhen},
  title   = {A Deep Generative Approach to Stratified Learning},
  journal = {arXiv preprint arXiv:2604.10650},
  year    = {2026}
}

@inproceedings{stanczuk2024diffusion,
  author    = {Stanczuk, Jan Pawel and Batzolis, Georgios and Deveney, Teo and Sch{\"o}nlieb, Carola-Bibiane},
  title     = {Diffusion Models Encode the Intrinsic Dimension of Data Manifolds},
  booktitle = {Forty-first International Conference on Machine Learning},
  year      = {2024}
}

@inproceedings{tempczyk2022lidl,
  author    = {Tempczyk, Piotr and Michaluk, Rafa{\l} and Garncarek, Lukasz and Spurek, Przemys{\l}aw and Tabor, Jacek and Golinski, Adam},
  title     = {{LIDL}: Local Intrinsic Dimension Estimation Using Approximate Likelihood},
  booktitle = {International Conference on Machine Learning},
  pages     = {21205--21231},
  publisher = {PMLR},
  year      = {2022}
}

@inproceedings{wang2021dimensionality,
  author    = {Wang, Qizhou and Erfani, Sarah M. and Leckie, Christopher and Houle, Michael E.},
  title     = {A Dimensionality-Driven Approach for Unsupervised Out-of-Distribution Detection},
  booktitle = {Proceedings of the 2021 SIAM International Conference on Data Mining},
  pages     = {118--126},
  publisher = {SIAM},
  year      = {2021}
}

@inproceedings{narayanan2010sample,
  author    = {Narayanan, Hariharan and Mitter, Sanjoy},
  title     = {Sample Complexity of Testing the Manifold Hypothesis},
  booktitle = {Advances in Neural Information Processing Systems},
  volume    = {23},
  year      = {2010}
}

@article{fefferman2016testing,
  author  = {Fefferman, Charles and Mitter, Sanjoy and Narayanan, Hariharan},
  title   = {Testing the Manifold Hypothesis},
  journal = {Journal of the American Mathematical Society},
  volume  = {29},
  number  = {4},
  pages   = {983--1049},
  year    = {2016}
}

@inproceedings{osada2026local,
  title = 	 {Local Hessian Spectral Filtering for Robust Intrinsic Dimension Estimation},
  author =       {Osada, Genki},
  booktitle = 	 {Proceedings of the 43rd International Conference on Machine Learning},
  pages = 	 {94864--94894},
  year = 	 {2026},
  volume = 	 {306},
  publisher =    {PMLR}
}

@article{jones2024diffusion,
  author  = {Jones, Iolo},
  title   = {Diffusion Geometry},
  journal = {arXiv preprint arXiv:2405.10858},
  year    = {2024}
}

@article{jones2026manifold,
  author  = {Jones, Iolo},
  title   = {Manifold Diffusion Geometry:
             Curvature, Tangent Spaces, and Dimension},
  journal = {SIAM Journal on Mathematics of Data Science},
  volume  = {8},
  number  = {2},
  pages   = {257--277},
  year    = {2026}
}

@article{seo2026nonparametric,
  title={Nonparametric In-Context Learning under Growing Geometric Complexity: Minimax Optimality and Local Geometry-Adaptivity of Transformers},
  author={Seo, Jaehee and Kim, Jisu},
  journal={arXiv preprint arXiv:2609.31458},
  year={2026}
}

@article{efron2011tweedie,
  author  = {Efron, Bradley},
  title   = {{Tweedie}'s formula and selection bias},
  journal = {Journal of the American Statistical Association},
  volume  = {106},
  number  = {496},
  pages   = {1602--1614},
  year    = {2011}
}

@article{dytso2022conditional,
  title={Conditional mean estimation in Gaussian noise: A meta derivative identity with applications},
  author={Dytso, Alex and Poor, H Vincent and Shitz, Shlomo Shamai},
  journal={IEEE Transactions on Information Theory},
  volume={69},
  number={3},
  pages={1883--1898},
  year={2023},
  publisher={IEEE}
}

\bigskip

\appendix

\counterwithin{theorem}{section}

\section*{\Huge Appendix}

\vspace{1em}

\section{Preliminaries}\label{app:preliminaries}

This section collects the analytic and geometric ingredients used in the
definition and analysis of the FLIPD field. The diffusion part describes
its representation through additive Gaussian smoothing, while the
manifold part specifies the regularity needed for uniform small-noise
expansions.

\subsection{Notation}

For $a,b\in\mathbb R$, write
$a\wedge b:=\min\{a,b\}$ and $a\vee b:=\max\{a,b\}$. For nonnegative quantities $A$ and $B$, write $A\lesssim B$ if
$A\le CB$ for a constant $C>0$ depending only on the fixed class
parameters, and write $A\asymp B$ if both $A\lesssim B$ and
$B\lesssim A$ hold. The symbol $\|\cdot\|$ denotes the Euclidean norm for vectors, and
$\|\cdot\|_{\mathrm{op}}$ denotes the operator norm.

Throughout the paper, $c,C>0$ denote generic constants whose values
may change from line to line. They are uniform in $n$, $\sigma$, $f$, and
the query point, with dependence on the fixed model-class parameters
specified in Section~\ref{sec:statistical-models}. In the lower-bound
arguments, they may additionally depend on the interior margin $\Delta_0$
and the fixed bump profiles introduced below.

\subsection{Diffusion and Additive Gaussian Smoothing}
\label{sec:forward-backward-processes}

The density $p_\sigma(x;f)$ in \eqref{eq:smoothed-density} is obtained by
convolving $P_f$ with $N(0,\sigma^2I_D)$. Its scale evolution and score
identity are given in \eqref{eq:gaussian-scale-fokker-planck} and
\eqref{eq:scale-score-identity}. Together with the definition
\eqref{eq:flipd-target}, they yield the representation
\eqref{eq:flipd-score} used throughout the analysis.

\subsection{Function Class and Geometry}
\label{sec:holder-manifolds}

\subsubsection{H\"older class on Euclidean domain.}
Let $U\subset\mathbb R^m$ be open, let $E$ be a finite-dimensional
normed vector space, and let $\gamma\in(0,1]$. For a map $h:U\to E$,
define
\begin{equation}
\label{eq:holder-seminorm}
 [h]_{C^{0,\gamma}(U)}
 :=
 \sup_{\substack{u,v\in U\\u\ne v}}
 \frac{\|h(u)-h(v)\|}{\|u-v\|^\gamma}.
\end{equation}
For $s\in\mathbb N_0$, the H\"older space $C^{s,\gamma}(U;E)$
consists of all $s$-times continuously differentiable maps
$h:U\to E$ such that
\begin{equation}
\label{eq:holder-norm}
 \|h\|_{C^{s,\gamma}(U)}
 :=
 \max_{0\le j\le s}\|D^j h\|_\infty
 +[D^s h]_{C^{0,\gamma}(U)}
 <\infty.
\end{equation}
Here $D^0h=h$, and derivatives are equipped with their induced
multilinear operator norms. When the target space $E$ is clear from
context, we simply write $C^{s,\gamma}(U)$.

For $a>0$, put
\[
    s_a:=\lceil a\rceil-1,
    \qquad
    \gamma_a:=a-s_a\in(0,1].
\]

\subsubsection{Reach and Regular Manifold Class.}\label{sec:reach-manifold}

We first recall the concept of reach (first proposed in \cite{federer1959curvature}) before defining the class of $d$-dimensional regular manifolds.

For a closed set $\cM\subset\mathbb R^D$, define the medial axis 
$Med(\cM)$ as the set of points having at least two nearest
points in $\cM$. Then, the reach of $\cM$ is defined by
\begin{equation}
\label{eq:reach-definition}
\tau_\cM
 :=\inf_{z\in \cM,\ y\in Med(\cM)}\|z-y\|,
\end{equation}
with the value $+\infty$ when $\operatorname{Med}(\cM)$ is empty.  Equivalently,
every point at distance less than $\tau_\cM$ from $\cM$ has a
unique nearest point in $\cM$. See Figure~\ref{fig:reach_illustration} for an illustration of the medial axis and the reach.

\begin{figure}
    \centering
    \includegraphics[width=0.5\linewidth]{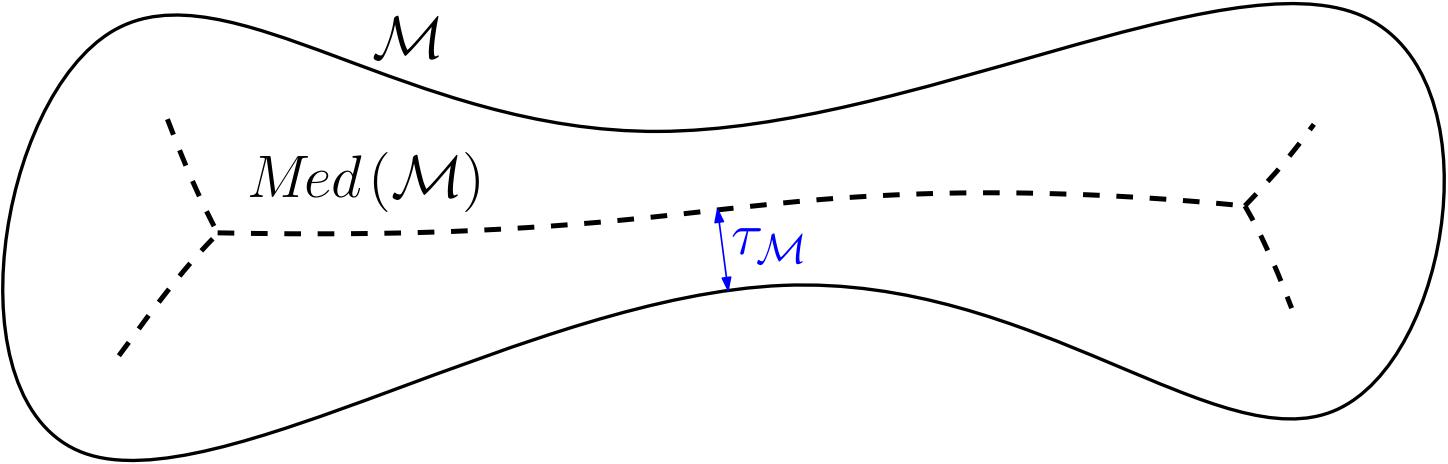}
    \caption{An illustration of the medial axis and the reach of the closed set $\cM$. A narrow bottleneck structure yields a small reach.
    }
    \label{fig:reach_illustration}
\end{figure}

We now define the regular submanifold class. Let $\beta>2$ and $r_\cM:=(4L_\cM)^{-1}$.
We say that a compact set $\cM\subset\mathbb R^D$ belongs to
$\mathcal C^\beta_{d, D, \tau, L_\cM}$ if it is a compact, connected,
boundaryless, embedded $d$-dimensional $C^{s_\beta,\gamma_\beta}$
submanifold with $\tau_\cM\ge \tau$, and, for every $x\in\cM$, there is a map
\begin{equation}
\label{eq:tangent-normal-chart-intrinsic}
 \begin{aligned}
 \Psi_x(v)=x+v+N_x(v),\qquad 
 v\in B_{T_x\cM}(0,r_\cM).
 \end{aligned}
\end{equation}
onto a relatively open neighborhood of $x$ in $\cM$, where
$N_x(v)\in N_x\cM:=(T_x\cM)^\perp$ and
\begin{align*}
 N_x(0)=0,\qquad
 DN_x(0)=0,\\
 \sup_x\sup_{2\le j\le s_\beta}
 \sup_{v\in B_{T_x\cM}(0,r_\cM)}
       \|D^jN_x(v)\|\le L_\cM,\\
 \sup_x[D^{s_\beta}N_x]_{C^{0,\gamma_\beta}}
 \le L_\cM.
\end{align*}
The radius and derivative bounds are part of the class definition and are
therefore uniform in the center $x$.

\noindent
Let $\vol_\cM$ denote the $d$-dimensional Hausdorff measure restricted to $\cM$.

\subsubsection{H\"older class on manifold.}
Suppose $\cM\in\mathcal C^\beta_{d, D, \tau, L_\cM}$ and
$\beta\ge\lceil\alpha\rceil+1$.  We define
\begin{equation}
\label{eq:manifold-holder-norm}
 \|f\|_{\mathcal H^\alpha(\cM)}
 :=\sup_{x\in\cM}
 \|f\circ\Psi_x\|_{C^{s_\alpha,\gamma_\alpha}
       (B_{T_x\cM}(0,r_\cM))},
\end{equation}
where $s_\alpha =
 \lceil \alpha\rceil-1$ and $\gamma_\alpha
 =\alpha-s_\alpha\in(0,1]$. The extra derivative in $\beta\ge\lceil\alpha\rceil+1$ makes the transition
maps between tangent--normal charts regular enough for the chain rule in
$C^{s_\alpha,\gamma_\alpha}$.  Consequently, this norm is equivalent, up to
constants depending only on the manifold-class parameters, to the norm
defined from any uniformly regular finite atlas.  We write
$\mathcal H^\alpha(\cM;L_f)$ for the ball on which the norm in
\eqref{eq:manifold-holder-norm} is at most $L_f$. Also, let $\eta:=(\alpha-2)\wedge1$.

\medskip

The formal local consequences of these definitions are stated in
Appendix~\ref{app:geometry}. In the main text we use them primarily as a
conceptual guide: locally, $\cM$ is well approximated by its tangent plane,
the Jacobian of the tangent--normal chart is uniformly controlled, and the
intrinsic volume of an ambient ball of radius $r$ centered on the manifold is
of order $r^d$.

\subsubsection{Role of the Geometric Conditions}
\label{sec:role-geometric-conditions}

The regular manifold model is used only through uniform local consequences. The positive reach condition prevents self-approach at scales below $\tau$
and ensures that points near $\cM$ have a well-defined nearest point on the
manifold. Such a condition is standard in statistical work on geometric
inference, including reach estimation, tangent-space and curvature
estimation, and minimax dimension estimation
\citep{aamari2019estimating, aamari2019nonasymptotic, Kim2019minimax}. This is
important for the FLIPD field because the Gaussian smoothing is performed
in the ambient space $\R^D$, whereas the data distribution is supported on
the lower-dimensional set $\cM$. Without a lower reach bound, an ambient
Gaussian ball centered at $x\in\cM$ could intersect geometrically unrelated
parts of the support at arbitrarily small scales, making the FLIPD field
unstable.

The tangent--normal charts in \eqref{eq:tangent-normal-chart-intrinsic}
make this locality explicit. Around each query point $x$, the manifold can
be written as a graph over the tangent space,
\[
    F_x(u)=x+U_xu+G_x(u),
    \qquad G_x(0)=DG_x(0)=0.
\]
The vanishing first derivative means that the leading-order local model is
the tangent plane. Curvature enters through the second fundamental form
$\mathrm{II}_x=D^2G_x(0)$ and therefore affects Gaussian integrals only at
second order. This is the geometric reason that the finite-scale bias in
Theorem~\ref{thm:small-noise-limits} is $O(\sigma^2)$. Accordingly, the second-order coefficient in the kernel-mass expansion
depends on both derivatives of the sampling density and the local extrinsic
geometry of $\cM$.

The H\"older condition on $f$ plays the analogous role for density
variation. In local coordinates,
\[
    f(F_x(u))=f(x)+D(f\circ F_x)(0)[u]
      +\frac12D^2(f\circ F_x)(0)[u,u]+\text{higher-order terms}.
\]
When this expansion is integrated against a centered Gaussian kernel, the
first-order term cancels by symmetry. The first nonzero density contribution
is therefore again quadratic in the scale. The assumptions
$\alpha>2$ and $\beta\ge\lceil\alpha\rceil+1$ ensure that this Taylor
expansion and the corresponding change of coordinates can be made uniformly
over all $x\in\cM$ and all $f\in\mathcal F_\alpha(\cM)$.

Together, these conditions yield the basic small-noise picture
\[
    Z_\sigma(f,x)
    \approx
    (2\pi)^{d/2}\sigma^d f(x),
    \qquad
    p_\sigma(x;f)
    \approx
    (2\pi)^{-(D-d)/2}\sigma^{d-D} f(x),
\]
uniformly over the model class. Thus the exponent of $\sigma$ in the
ambient smoothed density separates the intrinsic dimension $d$ from the
ambient normalization $D$. The FLIPD field reflects this exponent through
the logarithmic scale derivative and retains the lower-order corrections
at each fixed $\sigma>0$. Thus $T_\sigma(\cdot;f)$ depends on both density
variation and geometry, and Theorem~\ref{thm:minimax-lower} characterizes
its estimation difficulty.

The volume growth property formalized in Appendix~\ref{app:geometry} is also
the source of the effective sample size $n\sigma^d$ appearing in our minimax
lower bound. At
scale $\sigma$, a local Gaussian window sees intrinsic volume of order
$\sigma^d$, so among $n$ observations only about $n\sigma^d$ samples carry
substantial information about the FLIPD field near a fixed query
point. The lower bound in Section~\ref{sec:minimax-lower} formalizes this
heuristic by constructing many well-separated local perturbations of the
density on balls of radius comparable to $\sigma$.

\section{Local Geometry and Kernel Localization}
\label{app:geometry}

We use the uniformly
controlled tangent--normal charts supplied by the regular manifold class.
The following proposition collects the local geometric properties used in the
subsequent arguments, together with the Taylor expansion needed for the
small-noise analysis.

\begin{lemma}
\label{lem:reach-chord}
If $\tau_\cM\ge\tau>0$, then, for every $x,y\in\cM$,
\begin{equation}
\label{eq:reach-chord}
 \|\pi_{N_x\cM}(y-x)\|
 \le \frac{\|y-x\|^2}{2\tau}.
\end{equation}
\end{lemma}

\begin{proof}
Fix a unit vector $n\in N_x\cM$.  By the normal-tube characterization of
reach, $x$ is the unique nearest point of $x\pm tn$ on $\cM$ whenever
$0<t<\tau$; see, for example, the normal-bundle consequences of
\citet[Theorem~4.8]{federer1959curvature}.  Therefore
\[
 \|y-(x\pm tn)\|^2\ge t^2.
\]
Expanding the square for the two choices of sign gives
\[
 |\langle y-x,n\rangle|
 \le \frac{\|y-x\|^2}{2t}.
\]
Letting $t\uparrow\tau$ and taking the supremum over unit
$n\in N_x\cM$ proves \eqref{eq:reach-chord}.
\end{proof}

\medskip

\begin{proposition}[Uniform local geometry]
\label{prop:uniform-normal-atlas}
Assume $\cM\in\mathcal C^\beta_{d,D,\tau,L_\cM}$ and
$\beta\ge\lceil\alpha\rceil+1$ with $\alpha>2$. There exist constants
$r_0,c_0>0$ and $C_0<\infty$, depending only on $D,d,\alpha,\beta,L_\cM$, and
$\tau$, such that the following properties hold uniformly in $x\in\cM$.

Choose a linear isometry $U_x:\mathbb R^d\to T_x\cM$, and define
\[
F_x(u):=\Psi_x(U_xu)=x+U_xu+G_x(u),
\qquad
G_x(u):=N_x(U_xu),
\]
for $\|u\|<r_0$. Then $G_x(0)=DG_x(0)=0$, $F_x$ is a
$C^{s_\beta,\gamma_\beta}$ diffeomorphism onto its image, and
\begin{align}
 c_0\|u-v\|
 &\le \|F_x(u)-F_x(v)\|
 \le C_0\|u-v\|,
 \label{eq:atlas-bilipschitz}\\
 c_0
 &\le
 J_x(u):=
 \sqrt{\det\{I_d+DG_x(u)^\top DG_x(u)\}}
 \le C_0,
 \label{eq:atlas-jacobian}\\
 \max_{0\le j\le s_\beta}\|D^jG_x(u)\|
 &\le C_0,
 \qquad
 [D^{s_\beta}G_x]_{C^{0,\gamma_\beta}}
 \le C_0.
 \label{eq:atlas-regularity}
\end{align}

\noindent
Identifying $T_x\cM$ with $\mathbb R^d$ through $U_x$, write
$\mathrm{II}_x=D^2G_x(0)$ and
$Q_{3,x}(u)=D^3G_x(0)[u,u,u]/6$.
\begin{align}
 G_x(u)
 &=
 \frac12\mathrm{II}_x(u,u)+Q_{3,x}(u)+R_x(u),
 \qquad
 \|R_x(u)\|\le C_0\|u\|^{3+\eta},
 \label{eq:atlas-third-order-expansion}\\
 G_x(u)
 &=
 \frac12\mathrm{II}_x(u,u)+\widetilde R_x(u),
 \qquad
 \|\widetilde R_x(u)\|\le C_0\|u\|^3.
 \label{eq:atlas-II-expansion}
\end{align}

If two chart images overlap, let
\[
 \Omega_{xy}:=\{u\in B_d(0,r_0):F_x(u)\in F_y(B_d(0,r_0))\}
\]
and let $\Theta_{yx}:=F_y^{-1}\circ F_x$ on $\Omega_{xy}$.  Then
\begin{equation}
\label{eq:uniform-transition-bound}
 \|\Theta_{yx}\|_{C^{s_\alpha,\gamma_\alpha}(\Omega_{xy})}
 +\|\Theta_{xy}\|_{C^{s_\alpha,\gamma_\alpha}(\Omega_{yx})}
 \le C_0.
\end{equation}

\noindent
After decreasing $r_0$ by a class-dependent factor if necessary,
\begin{equation}
\label{eq:atlas-complement-separation}
 \inf_{y\in\cM\setminus F_x(B_d(0,r_0))}
 \|y-x\|
 \ge c_0r_0.
\end{equation}
Moreover, for $0<r\le c_0r_0/2$,
\begin{equation}
\label{eq:uniform-volume-growth}
 c_0r^d
 \le
 \vol_\cM\{\cM\cap B_D(x,r)\}
 \le
 C_0r^d.
\end{equation}
\end{proposition}

\medskip

\begin{proof}
Take
\[
 r_0\le \frac12(r_\cM\land\tau).
\]
From $DN_x(0)=0$, $\sup_x\sup_{2\le j\le s_\beta}
 \sup_{v\in B_{T_x\cM}(0,r_\cM)}
       \|D^jN_x(v)\|\le L_\cM$, and
$r_\cM=(4L_\cM)^{-1}$,
\begin{equation}
\label{eq:normal-map-first-derivative}
 \|DG_x(u)\|_{\mathrm{op}}
 \le \|u\|\sup_{\|w\|\le r_\cM}\|D^2G_x(w)\|_{\mathrm{op}}
 \le L_\cM r_\cM=\frac14.
\end{equation}
Because $U_x(u-v)\in T_x\cM$ and $G_x(u)-G_x(v)\in N_x\cM$ are
orthogonal,
\begin{align}
 \|F_x(u)-F_x(v)\|^2
 &=\|u-v\|^2+\|G_x(u)-G_x(v)\|^2, 
 \label{eq:graph-distance-identity}\\[1.5mm]
 \|u-v\|^2
 \le\|F_x(u)-F_x(v)\|^2
 &\le(1+1/16)\|u-v\|^2.
 \label{eq:graph-distance-bounds}
\end{align}
This proves \eqref{eq:atlas-bilipschitz}.  If
$s_1(u),\ldots,s_d(u)$ are the singular values of $DG_x(u)$, then
\begin{equation}
\label{eq:jacobian-singular-values}
 \begin{aligned}
 J_x(u)&=\prod_{j=1}^d\{1+s_j(u)^2\}^{1/2},\\
 1&\le J_x(u)\le(17/16)^{d/2}.
 \end{aligned}
\end{equation}
This gives \eqref{eq:atlas-jacobian}. The bounds for $j=0,1$ in
\eqref{eq:atlas-regularity} follow from $G_x(0)=DG_x(0)=0$ and
\eqref{eq:normal-map-first-derivative}. The bounds for $j\ge2$ and the
H\"older bound follow directly from the corresponding assumptions on $N_x$,
since $U_x$ is an isometry:
\[
 \sup_x\sup_{2\le j\le s_\beta}
 \sup_{v\in B_{T_x\cM}(0,r_\cM)}
       \|D^jN_x(v)\|\le L_\cM,
 \qquad
 \sup_x[D^{s_\beta}N_x]_{C^{0,\gamma_\beta}}\le L_\cM.
\]

Since $\beta\ge\lceil\alpha\rceil+1$ and $\alpha>2$, the maps $G_x$
have at least three derivatives. Moreover, $D^3G_x$ is uniformly
Lipschitz whenever $s_\beta\ge4$, while the borderline case
$s_\beta=3$ necessarily has $\gamma_\beta=1$. Hence, in all cases,
$D^3G_x$ is uniformly $\eta$-H\"older for
\[
\eta=(\alpha-2)\wedge1.
\]  
Taylor's formula with integral remainder gives
\begin{align*}
 R_x(u)
 &=\frac12\int_0^1(1-t)^2
   \{D^3G_x(tu)-D^3G_x(0)\}[u,u,u]\,\dd t,\\[2mm]
 \|R_x(u)\|
 &\le \frac{[D^3G_x]_{C^{0,\eta}}}{2}
       \|u\|^{3+\eta}
       \int_0^1(1-t)^2t^\eta\,\dd t
 \le C_0\|u\|^{3+\eta}.
\end{align*}
This proves \eqref{eq:atlas-third-order-expansion}. After decreasing
$r_0$ further so that $r_0\le1$, the uniform bound on $D^3G_x(0)$ and
$\|R_x(u)\|\le C_0\|u\|^{3+\eta}$ give
\[
\|Q_{3,x}(u)+R_x(u)\|\le C_0\|u\|^3,
\]
which proves \eqref{eq:atlas-II-expansion}.

We next verify the transition and separation assertions.  If
$z=F_y(v)$ lies in the image of the $y$-chart, tangent projection onto
$T_y\cM$ gives
\[
 v=U_y^\top(z-y).
\]
Consequently, on $\Omega_{xy}$,
\begin{equation}
\label{eq:transition-explicit}
 \Theta_{yx}(u)
 =U_y^\top\{x-y+U_xu+G_x(u)\}.
\end{equation}
The derivative and H\"older bounds in \eqref{eq:atlas-regularity}, together
with $\Theta_{yx}(\Omega_{xy})\subset B_d(0,r_0)$, give the first term in
\eqref{eq:uniform-transition-bound}.  Interchanging $x$ and $y$ gives the
same bound for the inverse transition.

For separation, set $\kappa_0:=4/\sqrt{17}$.  Suppose that
$y\in\cM$ satisfies $\|y-x\|<\delta$, where $\delta\le r_0/2$ will be
chosen below, and define
\[
 p:=U_x^\top(y-x),\qquad z:=F_x(p),\qquad q:=y-z.
\]
Then $\|p\|<r_0$, so $z$ is well defined, and $q\in N_x\cM$.  Moreover,
\begin{equation}
\label{eq:separation-q-upper}
 \|q\|\le\|y-x\|+\|z-x\|
 \le(1+C_0)\delta.
\end{equation}
The tangent space at $z=F_x(p)$ is the graph
\[
 T_z\cM
 =\{U_xh+DG_x(p)h:h\in\R^d\}.
\]
Since $q\in N_x\cM$ and $\|DG_x(p)\|_{\mathrm{op}}\le1/4$, minimizing
$\|q-U_xh-DG_x(p)h\|^2$ over $h$ gives
\begin{equation}
\label{eq:normal-graph-angle}
 \operatorname{dist}(q,T_z\cM)
 \ge\frac{\|q\|}{\sqrt{1+\|DG_x(p)\|_{\mathrm{op}}^2}}
 \ge\kappa_0\|q\|.
\end{equation}
On the other hand, Lemma~\ref{lem:reach-chord}, applied at $z$ to the chord
$y-z=q$, gives
\[
 \operatorname{dist}(q,T_z\cM)
 \le\frac{\|q\|^2}{2\tau}.
\]
Choose
\[
 \delta:=\min\left\{\frac{r_0}{2},
 \frac{\kappa_0\tau}{1+C_0}\right\}.
\]
If $q\ne0$, the last three displays imply both
$\kappa_0\le\|q\|/(2\tau)$ and
$\|q\|/(2\tau)\le\kappa_0/2$, a contradiction.  Thus $q=0$ and
$y=F_x(p)$.  Since $r_0\le\tau/2$, the number
\[
 c_{\mathrm{sep}}:=\min\left\{\frac12,
 \frac{2\kappa_0}{1+C_0}\right\}>0
\]
satisfies $c_{\mathrm{sep}}r_0\le\delta$.  Hence
\[
 \cM\cap B_D(x,c_{\mathrm{sep}}r_0)
 \subset F_x(B_d(0,r_0)).
\]
Decreasing the proposition's constant $c_0$ to
$c_0\wedge c_{\mathrm{sep}}$ proves
\eqref{eq:atlas-complement-separation} uniformly in $x$.

Finally, let $0<r\le c_0r_0/2$. The bi-Lipschitz bounds and the separation
from the complement imply
\begin{equation*}
 F_x\{B_d(0,r/C_0)\}
 \subset \cM\cap B_D(x,r)
 \subset F_x\{B_d(0,r/c_0)\}.
\end{equation*}
Applying the area formula and \eqref{eq:atlas-jacobian} to these two sets
proves \eqref{eq:uniform-volume-growth}.
\end{proof}

\medskip

\begin{lemma}
\label{lem:localization}
For every fixed $m\ge0$ and $p>0$, there exist constants
$0<c<C<\infty$ and $\sigma_0>0$, depending only on $m,p$ and the fixed
model-class parameters, such that, uniformly over
$f\in\mathcal F_\alpha(\cM)$, $x\in\cM$, and $0<\sigma\le\sigma_0$,
\begin{align}
c\sigma^d
\le Z_\sigma(f,x)
&\le C\sigma^d,
\label{eq:Z-localization}\\
\int_{\cM} K_{\sigma,x}(u)^p
\left(1+\frac{\|u-x\|}{\sigma}\right)^m
f(u)\,d\vol_\cM(u)
&\le C\sigma^d.
\label{eq:moment-localization}
\end{align}
\end{lemma}

\begin{proof}
Fix $m\ge0$ and $p>0$.  By
Proposition~\ref{prop:uniform-normal-atlas}, uniformly in $x$, the
chart $F_x$ is bi-Lipschitz with uniformly bounded Jacobian, and there
exists
\[
\delta_0:=c_0r_0>0
\]
such that
\[
\|u-x\|\ge\delta_0,
\qquad
u\in\cM\setminus F_x(B_d(0,r_0)).
\]
We decrease $\sigma_0$ if necessary so that $\sigma_0\le r_0$.

First consider the contribution from the chart.  Writing
$u=F_x(v)$ and using
\[
\|F_x(v)-x\|\ge c_0\|v\|,
\qquad
\|F_x(v)-x\|\le C_0\|v\|,
\]
together with $f\le c_+$ and $J_x(v)\le C_0$, gives
\begin{align*}
&\int_{F_x(B_d(0,r_0))}
K_{\sigma,x}(u)^p
\left(1+\frac{\|u-x\|}{\sigma}\right)^m
f(u)\,d\vol_\cM(u)
\\
&\qquad\le
C\int_{\|v\|<r_0}
\exp\left(-\frac{pc_0^2\|v\|^2}{2\sigma^2}\right)
\left(1+\frac{C_0\|v\|}{\sigma}\right)^m\,dv.
\end{align*}
With the change of variables $v=\sigma w$ and enlargement of the
integration domain,
\begin{align*}
&\int_{F_x(B_d(0,r_0))}
K_{\sigma,x}(u)^p
\left(1+\frac{\|u-x\|}{\sigma}\right)^m
f(u)\,d\vol_\cM(u)
\\
&\qquad\le
C\sigma^d
\int_{\mathbb R^d}
e^{-c\|w\|^2}(1+C\|w\|)^m\,dw
\le C\sigma^d.
\end{align*}

For the complement of the chart, set
$t=\|u-x\|/\sigma$.  Since $p>0$ is fixed,
\[
e^{-pt^2/2}(1+t)^m
\le C_{m,p}e^{-pt^2/4},
\qquad t\ge0.
\]
Hence, using $\|u-x\|\ge\delta_0$ and
$\int_\cM f\,d\vol_\cM=1$,
\begin{align*}
&\int_{\cM\setminus F_x(B_d(0,r_0))}
K_{\sigma,x}(u)^p
\left(1+\frac{\|u-x\|}{\sigma}\right)^m
f(u)\,d\vol_\cM(u)
\\
&\qquad\le
C\exp\left(-\frac{p\delta_0^2}{4\sigma^2}\right)
\le C\sigma^d
\end{align*}
for all sufficiently small $\sigma$.  This proves
\eqref{eq:moment-localization}.

Taking $m=0$ and $p=1$ gives the upper bound in
\eqref{eq:Z-localization}.  For the lower bound, since
$\sigma\le r_0$, we may integrate over $\|v\|\le\sigma$ in the local
chart.  Using $f\ge c_-$, $J_x(v)\ge c_0$, and
$\|F_x(v)-x\|\le C_0\|v\|$,
\begin{align*}
Z_\sigma(f,x)
&\ge
c
\int_{\|v\|\le\sigma}
\exp\left(
-\frac{C_0^2\|v\|^2}{2\sigma^2}
\right)\,dv
\\
&=
c\sigma^d
\int_{\|w\|\le1}
e^{-C_0^2\|w\|^2/2}\,dw
\ge c\sigma^d.
\end{align*}
This proves \eqref{eq:Z-localization}.
\end{proof}

\section{Proof of Theorem~\ref{thm:small-noise-limits}}\label{app:proof-bias}

The scale derivative defining the FLIPD field requires more than a pointwise
small-noise expansion.  We record the needed differentiated remainder
explicitly.

\begin{lemma}
\label{lem:differentiable-kernel-mass}
There exist $\sigma_0>0$ and, for every
$f\in\mathcal F_\alpha(\cM)$, a function $B_f:\cM\to\R$ such that
\[
\sup_{f\in\mathcal F_\alpha(\cM)}
\sup_{x\in\cM}|B_f(x)|\le C,
\]
and, uniformly over
$f\in\mathcal F_\alpha(\cM)$, $x\in\cM$, and
$0<r\le\sigma_0$,
\begin{equation}
\label{eq:differentiable-kernel-mass}
Z_r(f,x)
=
(2\pi)^{d/2}r^d
\{f(x)+r^2B_f(x)+R_r(f,x)\},
\end{equation}
where
\begin{equation}
\label{eq:differentiable-kernel-remainder}
\sup_{f,x}
\left\{
|R_r(f,x)|+
|r\partial_rR_r(f,x)|
\right\}
\le Cr^{2+\eta}.
\end{equation}
\end{lemma}

\medskip
\begin{proof}
Define the normalized mass
\[
 A_r(f,x):=(2\pi)^{-d/2}r^{-d}Z_r(f,x).
\]
We shall expand $A_r$ and its scale derivative separately.  This avoids
differentiating the expanding chart domain that appears after the change of
variables $v=rw$.  Differentiation under the original integral gives the
exact identity
\begin{equation}
\label{eq:normalized-mass-derivative}
 r\partial_r A_r(f,x)
 =(2\pi)^{-d/2}r^{-d}
 \int_\cM
 \left(\frac{\|u-x\|^2}{r^2}-d\right)
 K_{r,x}(u)f(u)\,\dd\vol_\cM(u).
\end{equation}
For each fixed $r>0$ the differentiation is justified by dominated
convergence; the uniform bounds needed as $r\downarrow0$ are established
below.

Work in the chart $F_x$ and set $h_{x,f}:=f\circ F_x$.  For
$w\in\R^d$, define
\[
 \mathsf A_{x,w}h:=\mathrm{II}_x(w,h),\qquad
 j_{2,x}(w):=\frac12\|\mathsf A_{x,w}\|_{\mathrm{HS}}^2,
 \qquad
 q_{4,x}(w):=\|\mathrm{II}_x(w,w)\|^2.
\]
Here $j_{2,x}$ and $q_{4,x}$ are homogeneous polynomials of degrees two
and four, respectively, with uniformly bounded coefficients.  Taylor's
formula and Proposition~\ref{prop:uniform-normal-atlas} give, uniformly for
$\|w\|\le r^{-1/3}$,
\begin{align}
 G_x(rw)
 &=\tfrac12r^2\mathrm{II}_x(w,w)
   +O(r^3\|w\|^3),
 \label{eq:B-G-expansion}\\
 DG_x(rw)
 &=r\mathsf A_{x,w}+O(r^2\|w\|^2),
 \label{eq:B-DG-expansion}\\
 J_x(rw)
 &=1+r^2j_{2,x}(w)+O\{r^3(1+\|w\|^M)\},
 \label{eq:B-J-expansion}\\
 h_{x,f}(rw)
 &=f(x)+rDh_{x,f}(0)[w]
   +\tfrac12r^2D^2h_{x,f}(0)[w,w]
   +O\{r^{2+\eta}(1+\|w\|^M)\},
 \label{eq:B-f-expansion}
\end{align}
for a fixed integer $M$.  The last remainder uses only the uniform
$C^{2,\eta}$ bound on $h_{x,f}$; in particular, no third derivative of the
density is being assumed.

By tangent--normal orthogonality,
\[
 \frac{\|F_x(rw)-x\|^2}{2r^2}
 =\frac{\|w\|^2}{2}+\frac{r^2}{8}q_{4,x}(w)
  +O\{r^3(1+\|w\|^M)\}.
\]
Using this display in the exponential and then multiplying
\eqref{eq:B-J-expansion} and \eqref{eq:B-f-expansion} yields
\begin{align}
 &e^{-\|F_x(rw)-x\|^2/(2r^2)}h_{x,f}(rw)J_x(rw)
 \nonumber\\
 &\quad=e^{-\|w\|^2/2}
 \{f(x)+rL_{1,x,f}(w)+r^2L_{2,x,f}(w)\}
 +r^{2+\eta}\mathcal E_{r,x,f}(w),
 \label{eq:B-integrand-expansion}
\end{align}
where
\begin{align}
 L_{1,x,f}(w)
 &:=Dh_{x,f}(0)[w],
 \label{eq:B-L1}\\
 L_{2,x,f}(w)
 &:=\tfrac12D^2h_{x,f}(0)[w,w]
   +f(x)j_{2,x}(w)-\frac{f(x)}8q_{4,x}(w),
 \label{eq:B-L2}
\end{align}
and
\begin{equation}
\label{eq:B-integrand-remainder}
 |\mathcal E_{r,x,f}(w)|
 \le C(1+\|w\|^M)e^{-c\|w\|^2}.
\end{equation}
Thus $L_{1,x,f}$ is odd, while $L_{2,x,f}$ is the sum of homogeneous
quadratic and quartic terms and is independent of $r$.

For the derivative identity \eqref{eq:normalized-mass-derivative}, the
additional weight has the expansion
\begin{equation}
\label{eq:B-derivative-weight}
 \frac{\|F_x(rw)-x\|^2}{r^2}-d
 =q_0(w)+\frac{r^2}{4}q_{4,x}(w)
  +O\{r^3(1+\|w\|^M)\},
 \qquad q_0(w):=\|w\|^2-d.
\end{equation}
Multiplying \eqref{eq:B-integrand-expansion} by
\eqref{eq:B-derivative-weight} gives
\begin{align}
 &\left(\frac{\|F_x(rw)-x\|^2}{r^2}-d\right)
 e^{-\|F_x(rw)-x\|^2/(2r^2)}h_{x,f}(rw)J_x(rw)
 \nonumber\\
 &\quad=e^{-\|w\|^2/2}
 \{f(x)q_0(w)+r q_0(w)L_{1,x,f}(w)\}
 \nonumber\\
 &\qquad+r^2e^{-\|w\|^2/2}
 \left\{q_0(w)L_{2,x,f}(w)+\frac{f(x)}4q_{4,x}(w)\right\}
 +r^{2+\eta}\widetilde{\mathcal E}_{r,x,f}(w),
 \label{eq:B-derivative-integrand-expansion}
\end{align}
with
\begin{equation}
\label{eq:B-derivative-remainder}
 |\widetilde{\mathcal E}_{r,x,f}(w)|
 \le C(1+\|w\|^M)e^{-c\|w\|^2}.
\end{equation}

We now identify the second-order coefficient.  Put
\begin{equation}
\label{eq:B-coefficient-definition}
 B_f(x):=(2\pi)^{-d/2}
 \int_{\R^d}e^{-\|w\|^2/2}L_{2,x,f}(w)\,\dd w.
\end{equation}
The coefficient bounds above imply $\sup_{f,x}|B_f(x)|\le C$.
Moreover, if $P_k$ is a homogeneous polynomial of degree $k$ and
$W\sim N(0,I_d)$, Gaussian integration by parts gives
\begin{equation}
\label{eq:gaussian-homogeneous-identity}
 \E\{(\|W\|^2-d)P_k(W)\}=k\E P_k(W).
\end{equation}
Write $L_{2,x,f}=P_{2,x,f}-f(x)q_{4,x}/8$, where $P_{2,x,f}$ is
homogeneous quadratic.  Applying \eqref{eq:gaussian-homogeneous-identity}
with $k=2$ and $k=4$ shows that
\begin{align}
 &(2\pi)^{-d/2}\int_{\R^d}e^{-\|w\|^2/2}
 \left\{q_0(w)L_{2,x,f}(w)+\frac{f(x)}4q_{4,x}(w)\right\}\,\dd w
 \nonumber\\
 &\qquad=2B_f(x).
 \label{eq:B-derivative-coefficient}
\end{align}
Also, the zeroth-order term in
\eqref{eq:B-derivative-integrand-expansion} integrates to zero, and its
first-order term is odd.

It remains to check that discarding the tails is legitimate for both
expansions.  On the chart portion $\|w\|>r^{-1/3}$,
\eqref{eq:graph-distance-bounds} bounds the mass integrand, and the
derivative integrand in \eqref{eq:normalized-mass-derivative}, by a fixed
polynomial times $e^{-\|w\|^2/2}$.  Outside the chart,
\eqref{eq:atlas-complement-separation} gives $\|u-x\|\ge\delta_0>0$.
Using
\[
 (1+t^2)e^{-t^2/2}\le Ce^{-t^2/4},\qquad t\ge0,
\]
and $\int_\cM f\,\dd\vol_\cM=1$ shows that the normalized mass and
derivative tails are bounded respectively by
\[
 Cr^{-d}e^{-cr^{-2/3}}+Cr^{-d}e^{-c/r^2}.
\]
This is $O(r^N)$ for every fixed $N>0$, uniformly in $f$ and $x$.

We may therefore integrate \eqref{eq:B-integrand-expansion} over
$\R^d$, at a cost $O(r^{2+\eta})$, to obtain
\begin{equation}
\label{eq:B-normalized-mass-expansion}
 A_r(f,x)=f(x)+r^2B_f(x)+O(r^{2+\eta}).
\end{equation}
Likewise, \eqref{eq:normalized-mass-derivative},
\eqref{eq:B-derivative-integrand-expansion}, and
\eqref{eq:B-derivative-coefficient} give
\begin{equation}
\label{eq:B-normalized-derivative-expansion}
 r\partial_rA_r(f,x)=2r^2B_f(x)+O(r^{2+\eta}).
\end{equation}
Define $R_r(f,x):=A_r(f,x)-f(x)-r^2B_f(x)$.  Equations
\eqref{eq:B-normalized-mass-expansion} and
\eqref{eq:B-normalized-derivative-expansion} imply
\[
 |R_r(f,x)|+|r\partial_rR_r(f,x)|\le Cr^{2+\eta}.
\]
Multiplying the definition of $A_r$ by $(2\pi)^{d/2}r^d$ completes the
proof of \eqref{eq:differentiable-kernel-mass}--
\eqref{eq:differentiable-kernel-remainder}.
\end{proof}

\subsection{Main Proof of Theorem~\ref{thm:small-noise-limits}}

\begin{proof}
By Lemma~\ref{lem:differentiable-kernel-mass},
\[
Z_\sigma(f,x)
=
(2\pi)^{d/2}\sigma^d A_\sigma(f,x),
\]
where
\[
A_\sigma(f,x)
:=
f(x)+\sigma^2B_f(x)+R_\sigma(f,x).
\]
Since $f\ge c_-$ and
\[
|B_f(x)|\le C,
\qquad
|R_\sigma(f,x)|\le C\sigma^{2+\eta},
\]
after decreasing $\sigma_0$ if necessary,
\[
A_\sigma(f,x)\ge \frac{c_-}{2}
\]
uniformly over $f$, $x$, and $0<\sigma\le\sigma_0$.

Therefore
\begin{align*}
\sigma\partial_\sigma\log Z_\sigma(f,x)
&=
d+\sigma\partial_\sigma\log A_\sigma(f,x)\\[0.2cm]
&=
d+
\frac{
2\sigma^2B_f(x)
+\sigma\partial_\sigma R_\sigma(f,x)
}{
f(x)+\sigma^2B_f(x)+R_\sigma(f,x)
}.
\end{align*}
Using
\[
|B_f(x)|\le C,
\qquad
|\sigma\partial_\sigma R_\sigma(f,x)|
\le C\sigma^{2+\eta},
\]
we obtain uniformly
\[
\left|
\sigma\partial_\sigma\log Z_\sigma(f,x)-d
\right|
\le C\sigma^2.
\]
By the definition of the FLIPD field in
\eqref{eq:flipd-target}, this proves
\eqref{eq:flipd-limit}.
\end{proof}

\section{Proof of Theorem~\ref{thm:minimax-lower}}
\label{app:minimax-lower}

\subsection{Assouad's Scheme}
\label{app:assouad}

We first fix the uniformity conventions used throughout this appendix.  Let
$r_0,c_0,C_0$ be the constants in
Proposition~\ref{prop:uniform-normal-atlas}, and let
\[
 I:=[1/2,2].
\]
All distances between points of $\cM$ are ambient Euclidean distances.
Constants may depend on the fixed model-class parameters, the interior
margin $\Delta_0$, and the fixed bump profiles. They are uniform in
$a,s,n,\theta$, the number of bump centers, and their signs, and may change
from line to line. Whenever a statement holds for $0<s\le s_0$, the
constant $s_0$ is uniform in all the variables displayed in that statement.
For every integrable signed function $g$ on $\cM$, we use the linear extension
\[
 Z_\sigma(g,x):=
 \int_\cM K_{\sigma,x}(u)g(u)\,\dd\vol_\cM(u).
\]

Fix $x_\star\in\cM$ and the compact coordinate subpatch
\begin{equation}
\label{eq:packing-patch}
 K:=F_{x_\star}\!\left(\overline{B_d(0,r_0/8)}\right).
\end{equation}
By the area formula and \eqref{eq:atlas-jacobian},
$\vol_\cM(K)\ge c>0$.  All packing centers used below will lie in $K$.

The lower-bound construction for the FLIPD field starts from a nonzero smooth profile
$\psi\in C_c^\infty(\R^d)$ satisfying $\int_{\R^d}\psi=0$.  The next lemma
transfers this profile to $\cM$ while preserving zero mass and keeping all
constants uniform in the center.

\begin{lemma}
\label{lem:scaled-bump}
Fix a nonzero $\psi\in C_c^\infty(\R^d)$ with
$\int_{\R^d}\psi=0$.  There exist a nonnegative
$\chi\in C_c^\infty(\R^d)$ with $\int_{\R^d}\chi>0$, a radius $R_b<\infty$,
and constants $s_b,c_b,C_b>0$ such that
\[
 \supp(\psi)\cup\supp(\chi)\subset B_d(0,R_b)
\]
and the following holds.  For every $a\in\cM$ and $0<s\le s_b$, define
\begin{equation}
\label{eq:manifold-bump-definition}
 c_{a,s}:=
 \frac{\int_{\R^d}\psi(w)J_a(sw)\,\dd w}
      {\int_{\R^d}\chi(w)J_a(sw)\,\dd w},
 \qquad
 b_{a,s}(F_a(sw)):=\psi(w)-c_{a,s}\chi(w),
\end{equation}
on the scaled support, and set $b_{a,s}=0$ elsewhere on $\cM$.  Then the
zero extension belongs to $\mathcal H^\alpha(\cM)$ and, uniformly over
$a\in\cM$ and $0<s\le s_b$,
\begin{align}
 \int_\cM b_{a,s}\,\dd\vol_\cM&=0,
 \notag\\
 |c_{a,s}|&\le C_b s^2,
 \label{eq:bump-mass}\\
 \|b_{a,s}\|_\infty&\le C_b,
 \notag\\
 \|b_{a,s}\|_{\mathcal H^\alpha(\cM)}&\le C_b s^{-\alpha},
 \label{eq:bump-holder}\\
 c_b s^d
 \le\int_\cM b_{a,s}^2\,\dd\vol_\cM
 &\le C_b s^d.
 \label{eq:bump-L2}
\end{align}
Moreover,
\begin{equation}
\label{eq:bump-support}
 \supp(b_{a,s})\subset \cM\cap B_D(a,C_b s).
\end{equation}
\end{lemma}

\begin{proof}
Choose a nonnegative $\chi\in C_c^\infty(\R^d)$ with positive integral and
then choose $R_b$ so that both profiles are supported in $B_d(0,R_b)$.
Decrease $s_b$ so that $s_bR_b<r_0/4$.  Then every point $F_a(sw)$ used in
\eqref{eq:manifold-bump-definition} lies strictly inside the chart domain.
Because the profiles vanish on a neighborhood of the boundary of their
common supporting ball, extension by zero defines an
$\mathcal H^\alpha(\cM)$ function.

The area formula and the change of variables $v=sw$ give
\begin{align*}
 \int_\cM b_{a,s}\,\dd\vol_\cM
 &=s^d\left\{
 \int\psi(w)J_a(sw)\,\dd w
 -c_{a,s}\int\chi(w)J_a(sw)\,\dd w
 \right\}=0.
\end{align*}
We next quantify the correction $c_{a,s}$.  From
$DG_a(0)=0$ and the uniform bound on $D^2G_a$,
\[
 \|DG_a(sw)\|\le Cs\|w\|.
\]
Since
$J_a(sw)=\sqrt{\det\{I_d+DG_a(sw)^\top DG_a(sw)\}}$, it follows, uniformly
for $w\in B_d(0,R_b)$, that
\begin{equation}
\label{eq:jacobian-quadratic-bound}
 |J_a(sw)-1|\le Cs^2\|w\|^2.
\end{equation}
Using $\int\psi=0$, the numerator in \eqref{eq:manifold-bump-definition}
is therefore bounded by $Cs^2$.  The denominator is bounded below by
$c_0\int\chi>0$ by \eqref{eq:atlas-jacobian}.  This proves
\eqref{eq:bump-mass}, and the sup-norm bound follows immediately.

For the H\"older bound, write, in the chart centered at $a$,
\[
 (b_{a,s}\circ F_a)(v)
 =\psi(v/s)-c_{a,s}\chi(v/s),
\]
where both profiles are understood to be zero off their supports.  For
$0\le k\le s_\alpha$, scaling gives
\[
 \|D^k(b_{a,s}\circ F_a)\|_\infty\le Cs^{-k},
 \qquad
 [D^{s_\alpha}(b_{a,s}\circ F_a)]_{C^{0,\gamma_\alpha}}
 \le Cs^{-\alpha}.
\]
The transition maps satisfy \eqref{eq:uniform-transition-bound}.
To verify the bound on each full chart used in
\eqref{eq:manifold-holder-norm}, fix $x\in\cM$ and write
$h_x:=b_{a,s}\circ\Psi_x$. On the overlap with the $a$-chart, its
coordinate transition is
$F_a^{-1}\circ\Psi_x(v)=U_a^\top\{\Psi_x(v)-a\}$.
The manifold regularity bounds control derivatives of this map through
order $s_\alpha+1$, since $\beta\ge\lceil\alpha\rceil+1$.
The scaled profiles vanish near the boundary of $B_d(0,sR_b)$, whose
closure lies inside the $a$-chart domain. Thus the chain rule and zero
extension give, uniformly in $x$,
\[
 \|D^k h_x\|_\infty\le Cs^{-k},\qquad 0\le k\le s_\alpha+1.
\]
The domain of $\Psi_x$ is convex. Hence, for any two points $u,v$ in that
domain,
\[
 \|D^{s_\alpha}h_x(u)-D^{s_\alpha}h_x(v)\|
 \le C\min\{s^{-s_\alpha},s^{-s_\alpha-1}\|u-v\|\}
 \le Cs^{-\alpha}\|u-v\|^{\gamma_\alpha}.
\]
Taking the supremum over $x$ proves \eqref{eq:bump-holder}, including
pairs of points that cross the boundary of the bump support.

Finally, uniformly in $a$,
\begin{align*}
 s^{-d}\int_\cM b_{a,s}^2\,\dd\vol_\cM
 &=\int_{\R^d}|\psi(w)-c_{a,s}\chi(w)|^2J_a(sw)\,\dd w
 \longrightarrow \|\psi\|_{L^2(\R^d)}^2.
\end{align*}
The convergence is uniform by \eqref{eq:bump-mass} and
\eqref{eq:jacobian-quadratic-bound}.  Decreasing $s_b$ once more makes the
last display lie between two fixed positive constants, proving
\eqref{eq:bump-L2}.  The support bound follows from
\eqref{eq:atlas-bilipschitz}.
\end{proof}

The following elementary consequence is needed because the Assouad
hypercube contains order $s^{-d}$ bumps.  Applying the triangle inequality
to their H\"older norms would introduce an incorrect factor of order
$s^{-d}$; separation of the supports avoids that loss.

\begin{lemma}
\label{lem:separated-bump-holder}
Let $a_1,\ldots,a_m\in\cM$ satisfy
$\|a_j-a_k\|\ge L_0s$ for $j\ne k$, where $L_0>2C_b$, and write
$b_j=b_{a_j,s}$.  For every choice of coefficients
$|\omega_j|\le1$,
\begin{equation}
\label{eq:separated-bump-holder}
 \left\|\sum_{j=1}^m\omega_jb_j\right\|_{\mathcal H^\alpha(\cM)}
 \le Cs^{-\alpha},
\end{equation}
where $C$ is independent of $m$ and of the coefficients.
\end{lemma}

\begin{proof}
By \eqref{eq:bump-support}, distinct supports are separated by at least
$(L_0-2C_b)s$.  In any uniformly regular chart, at most one summand is
nonzero at a given point, so all derivative sup norms of order
$k\le s_\alpha$ are bounded by $Cs^{-k}$.

It remains to control the $\gamma_\alpha$-H\"older seminorm of the top
derivative.  If two chart points meet the same support, or if one point is
outside all supports and the other meets a single support, the single-bump
bound in Lemma~\ref{lem:scaled-bump} applies to the zero extension.  If the
two points meet different supports, their manifold, ambient, and chart
distances are all bounded below by $cs$ by
\eqref{eq:atlas-bilipschitz}.  Hence
\[
 \frac{\|D^{s_\alpha}b_j(u)-D^{s_\alpha}b_k(v)\|}
      {\|u-v\|^{\gamma_\alpha}}
 \le \frac{Cs^{-s_\alpha}}{(cs)^{\gamma_\alpha}}
 \le Cs^{-\alpha}.
\]
The transition estimate \eqref{eq:uniform-transition-bound} then gives
\eqref{eq:separated-bump-holder} in the norm
\eqref{eq:manifold-holder-norm}.
\end{proof}

\begin{lemma}\label{lem:adjacent-KL}
Let $g$ be a density on $\cM$ with $g\ge g_*>0$.  Suppose
$f_\pm=g\pm\theta b_{a,s}$ and
$|\theta|C_b\le g_*/2$.  Then $f_\pm$ are densities and
\begin{equation}
\label{eq:adjacent-KL}
 \KL(P_{f_+}^{\otimes n},P_{f_-}^{\otimes n})
 \le Cn\theta^2s^d,
\end{equation}
where $C$ depends on $g_*^{-1}$ but not on $a,s,n$, or $\theta$.
\end{lemma}

\begin{proof}
The zero-mass property in Lemma~\ref{lem:scaled-bump} shows that both
functions integrate to one.  Moreover, $f_-\ge g_*/2$.  Therefore
\begin{align*}
 \KL(P_{f_+},P_{f_-})
 &\le \chi^2(P_{f_+},P_{f_-})
 =\int_\cM\frac{(f_+-f_-)^2}{f_-}\,\dd\vol_\cM\\
 &\le \frac{8\theta^2}{g_*}
       \int_\cM b_{a,s}^2\,\dd\vol_\cM
 \le C\theta^2s^d.
\end{align*}
Tensorization of the Kullback--Leibler divergence proves
\eqref{eq:adjacent-KL}.
\end{proof}

\begin{lemma}
\label{lem:two-point}
Let $Q_0,Q_1$ be probability measures on the same measurable space, let
$H$ be a separable Hilbert space, and let $u_0,u_1\in H$.  For every
measurable $H$-valued estimator $\widehat u$,
\begin{equation}
\label{eq:two-point}
 \E_{Q_0}\|\widehat u-u_0\|_H^2
 +\E_{Q_1}\|\widehat u-u_1\|_H^2
 \ge\frac{1-\TV(Q_0,Q_1)}2\|u_0-u_1\|_H^2,
\end{equation}
where $\TV(Q_0,Q_1)=\sup_A|Q_0(A)-Q_1(A)|$.
\end{lemma}

\begin{proof}
Let $\mu=Q_0+Q_1$ and let $Q_0\wedge Q_1$ be the measure with
$\mu$-density
$\min\{\dd Q_0/\dd\mu,\dd Q_1/\dd\mu\}$.  Its total mass is
$1-\TV(Q_0,Q_1)$.  Integrating
\[
 \|v-u_0\|_H^2+\|v-u_1\|_H^2
 \ge \frac12\|u_0-u_1\|_H^2
\]
with respect to $Q_0\wedge Q_1$ proves the claim.
\end{proof}

We now quantify the fact that a Gaussian kernel centered near one bump sees
only an exponentially small contribution from the other bumps.  The
operator $r\partial_r$ in the next lemma acts only on the kernel scale
$rs$; the bump centers, supports, and the base scale $s$ are held fixed.

\begin{lemma}
\label{lem:shell-decoupling}
Fix integers $\ell,m\ge0$ and constants $C_q,C_b<\infty$.  There exist
$L_*>0$ and $c,C>0$ with the following property.  Let $L_0\ge L_*$ be
fixed.  For all sufficiently small $s$ (the upper bound may depend on the
fixed $L_0$), let $a_1,\ldots,a_M\in K$ be ambiently
$L_0s$-separated, let $b_k=b_{a_k,s}$, and suppose
$\|x-a_j\|\le C_qs$.  Then, uniformly in $r\in I$,
\begin{equation}
\label{eq:shell-sum}
 \sum_{k\ne j}\int_\cM
 \left|(r\partial_r)^\ell K_{rs,x}(u)\right|
 \left(1+\frac{\|u-x\|}{s}\right)^m
 |b_k(u)|\,\dd\vol_\cM(u)
 \le Cs^d\delta_{\ell,m}(L_0),
\end{equation}
where
\begin{equation}
\label{eq:shell-decay}
 \delta_{\ell,m}(L_0)
 :=\sum_{q\ge1}(q+1)^d
 \{1+(q+1)L_0\}^{m+2\ell}
 e^{-c(q+1)^2L_0^2}.
\end{equation}
In particular, $\delta_{\ell,m}(L_0)<\infty$ and
$\delta_{\ell,m}(L_0)\to0$ as $L_0\to\infty$.
\end{lemma}

\begin{proof}
For $q\ge1$, let
\[
 \mathcal A_q
 :=\{k\ne j:qL_0s\le\|a_k-a_j\|<(q+1)L_0s\}.
\]
The ambient balls
$\cM\cap B_D(a_k,L_0s/3)$, $k\in\mathcal A_q$, are disjoint.  If
$(q+2)L_0s$ is below the volume-growth radius in
\eqref{eq:uniform-volume-growth}, comparison with the ball centered at
$a_j$ gives $|\mathcal A_q|\le C(q+1)^d$.  If
$(q+2)L_0s$ exceeds that radius, compactness and the lower volume bound give
$|\mathcal A_q|\le C(L_0s)^{-d}$, whereas
$q+2\ge c(L_0s)^{-1}$; the same estimate follows.  Thus
\begin{equation}
\label{eq:shell-cardinality}
 |\mathcal A_q|\le C(q+1)^d
\end{equation}
for every $q\ge1$.  Here $s$ has been chosen small enough that
$L_0s/3$ lies in the range of \eqref{eq:uniform-volume-growth}.

Choose $L_*$ so that $L_0\ge4(C_b+C_q)$.  If
$k\in\mathcal A_q$ and $u\in\supp(b_k)$, then
\[
 \|u-x\|
 \ge qL_0s-(C_b+C_q)s
 \ge c(q+1)L_0s.
\]
For $r\in I$, repeated differentiation gives
\[
 \left|(r\partial_r)^\ell K_{rs,x}(u)\right|
 \le C_\ell
 \left(1+\frac{\|u-x\|}{s}\right)^{2\ell}
 \exp\!\left(-c\frac{\|u-x\|^2}{s^2}\right).
\]
Furthermore, Lemma~\ref{lem:scaled-bump} and
\eqref{eq:uniform-volume-growth} give
$\|b_k\|_\infty\le C$ and
$\vol_\cM(\supp b_k)\le Cs^d$.  Hence the total contribution from
$\mathcal A_q$ is at most
\[
 Cs^d(q+1)^d
 \{1+(q+1)L_0\}^{m+2\ell}
 e^{-c(q+1)^2L_0^2}.
\]
Summation in $q$ proves \eqref{eq:shell-sum}.  The final assertion follows
by dominated convergence, for example after restricting to $L_0\ge L_*$.
\end{proof}

For later use, set
\begin{equation}
\label{eq:delta-L0-definition}
 \delta(L_0):=\delta_{0,0}(L_0)+\delta_{1,0}(L_0).
\end{equation}

The next proposition is the precise Assouad reduction used in the proof.  In
contrast with an informal ``remote-bump stability'' assumption, its
hypothesis states exactly the separation required for every adjacent edge of
the hypercube.

\begin{proposition}[Assouad reduction for FLIPD]
\label{prop:assouad-reduction}
Suppose that there exist constants satisfying
\[
 L_0,s_A,c_A,C_q>0,\qquad \kappa\ge0,
 \qquad L_0>2(C_b+C_q).
\]
Assume that the following holds for every
$0<s\le s_A$.  For every finite ambiently $L_0s$-separated family
$a_1,\ldots,a_m\in K$, put $b_j=b_{a_j,s}$ and
\[
 f_\omega:=f_0+\theta\sum_{j=1}^m\omega_jb_j,
 \qquad \omega\in\{-1,1\}^m,
\]
where $0<\theta\le s^2$.  Assume that these functions are positive
densities and that there are measurable sets
\[
 U_{a_j,s}\subset \cM\cap B_D(a_j,C_qs)
\]
such that, for every $\omega$ and every $j$,
\begin{equation}
\label{eq:generic-one-bump}
 \int_{U_{a_j,s}}
 \left|T_s(x;f_\omega)-T_s(x;f_{\omega^{(j)}})\right|^2
 \,\dd\vol_\cM(x)
 \ge c_A\theta^2s^{d+2\kappa},
\end{equation}
where $\omega^{(j)}$ is obtained from $\omega$ by flipping its $j$th
coordinate.  Then, for every
$n$ and $s$ satisfying $h_n\le s\le s_A\wedge1$,
\begin{equation}
\label{eq:generic-lower}
 \inf_{\widehat T}\sup_{P_f\in\cP_\alpha(\cM,d)}
 \mathcal R_{n,s}(\widehat T,f)
 \ge c\frac{s^{2\kappa}}{ns^d}.
\end{equation}
\end{proposition}

\begin{proof}
Fix $s\le s_A\wedge1$ and take a maximal $L_0s$-separated subset
$\{a_1,\ldots,a_{M_s}\}$ of $K$.  Maximality implies that the ambient
$L_0s$-balls centered at these points cover $K$.  The upper volume-growth
bound therefore yields
\[
 \vol_\cM(K)\le CM_s(L_0s)^d,
\]
and hence $M_s\ge cs^{-d}$.  Conversely, the ambient balls of radius
$L_0s/3$ centered at the $a_j$ are disjoint.  Their volumes are at least
$c(L_0s)^d$, so compactness of $\cM$ gives $M_s\le Cs^{-d}$.  Thus
\begin{equation}
\label{eq:packing-cardinality}
 cs^{-d}\le M_s\le Cs^{-d}.
\end{equation}
We take $s_A$ small enough that all radii used here lie in the range of
\eqref{eq:uniform-volume-growth}.  The assumed inequality
$L_0>2(C_b+C_q)$ ensures that both the bump supports and the query regions
are pairwise disjoint; the value of $L_0$ is not changed in the proof.

Set
\begin{equation}
\label{eq:assouad-amplitude}
 \theta:=\varepsilon(ns^d)^{-1/2},
 \qquad
 f_\omega:=f_0+\theta\sum_{j=1}^{M_s}\omega_jb_j,
 \qquad
 \omega\in\{-1,1\}^{M_s}.
\end{equation}
Because $s\ge h_n=n^{-1/(2\alpha+d)}$,
\begin{equation}
\label{eq:holder-amplitude-equivalence}
 (ns^d)^{-1/2}\le s^\alpha.
\end{equation}
Lemmas~\ref{lem:scaled-bump} and
\ref{lem:separated-bump-holder} give
\begin{align}
 \int_\cM f_\omega\,\dd\vol_\cM&=1,
 \notag\\
 \|f_\omega-f_0\|_\infty&\le C\varepsilon,
 \notag\\
 \|f_\omega-f_0\|_{\mathcal H^\alpha(\cM)}
 &\le C\theta s^{-\alpha}\le C\varepsilon.
 \label{eq:hypercube-holder}
\end{align}
Using the positive interior margin $\Delta_0$ defined in
Section~\ref{sec:statistical-models}, choose $\varepsilon\le1$ so that
\begin{equation}
\label{eq:epsilon-model-and-KL}
 C\varepsilon\le\Delta_0/2,
 \qquad C_b\varepsilon\le c_-/2.
\end{equation}
Then
$P_{f_\omega}\in\cP_\alpha(\cM,d)$ for every $\omega$.  Moreover,
\eqref{eq:holder-amplitude-equivalence}, $s\le1$, and $\alpha>2$ imply
\begin{equation}
\label{eq:theta-below-squared-scale}
 \theta\le\varepsilon s^\alpha\le s^2.
\end{equation}

For adjacent sign vectors $\omega$ and $\omega^{(j)}$, apply
Lemma~\ref{lem:adjacent-KL} with
$g=f_0+\theta\sum_{k\ne j}\omega_kb_k$.  The supports are disjoint.  On
$\supp(b_j)$ one has $g=f_0$, while off $\supp(b_j)$ the function $g$
coincides with both adjacent densities.  Hence $g\ge c_-$.  Also,
$\theta\le\varepsilon$ and \eqref{eq:epsilon-model-and-KL} imply
$\theta C_b\le c_-/2$.  Thus all hypotheses of
Lemma~\ref{lem:adjacent-KL} are satisfied, and
\begin{equation}
\label{eq:hypercube-KL}
 \KL(P_{f_\omega}^{\otimes n},
     P_{f_{\omega^{(j)}}}^{\otimes n})
 \le Cn\theta^2s^d=C\varepsilon^2.
\end{equation}
Pinsker's inequality gives
\[
 \TV(P_{f_\omega}^{\otimes n},
     P_{f_{\omega^{(j)}}}^{\otimes n})
 \le \sqrt{\tfrac12 C\varepsilon^2}.
\]
Decrease $\varepsilon$ once more so that the last quantity is at most
$1/2$.

Let $H_j=L^2(U_{a_j,s},\dd\vol_\cM)$.  Because the query regions are
disjoint and $f_\omega\ge c_-$, every estimator satisfies
\begin{align*}
 \sup_\omega\mathcal R_{n,s}(\widehat T,f_\omega)
 &\ge 2^{-M_s}\sum_\omega
        \mathcal R_{n,s}(\widehat T,f_\omega)\\
 &\ge c_-2^{-M_s}\sum_{j=1}^{M_s}\sum_\omega
 \E_\omega
 \|\widehat T-T_s(\cdot;f_\omega)\|_{H_j}^2.
\end{align*}
For a fixed $j$, the sign vectors form $2^{M_s-1}$ unordered adjacent
pairs.  Apply Lemma~\ref{lem:two-point} to each pair.  By
\eqref{eq:generic-one-bump} and the total-variation bound, each pair
contributes at least
$c\theta^2s^{d+2\kappa}$.  Consequently,
\begin{align*}
 \sup_\omega\mathcal R_{n,s}(\widehat T,f_\omega)
 &\ge cM_s\theta^2s^{d+2\kappa}\\
 &\ge c\theta^2s^{2\kappa}
 =c\varepsilon^2\frac{s^{2\kappa}}{ns^d}.
\end{align*}
The choice of $\varepsilon$ is uniform in $n$ and $s$ and may depend on
$\Delta_0$. Taking the infimum over $\widehat T$ proves
\eqref{eq:generic-lower} with the stated constant dependence.
\end{proof}

\subsection{Auxiliary Lemmas}
\label{app:ratio-expansion}

Let
\[
 \gamma_\tau(z):=(2\pi\tau^2)^{-d/2}
 \exp\!\left(-\frac{\|z\|^2}{2\tau^2}\right).
\]

\begin{lemma}\label{lem:ratio-expansion}
Fix a smooth compactly supported mean-zero profile $\psi$ and construct
$b_{a,s}$ as in Lemma~\ref{lem:scaled-bump}.  For every fixed $R<\infty$
there exist $s_R,C_R>0$ and functions $\rho_{a,s}(z,\tau)$ such that, for
all $a\in\cM$, $0<s\le s_R$, $\|z\|\le R$, and $\tau\in I$,
\begin{equation}
\label{eq:ratio-expansion}
 \frac{Z_{\tau s}(b_{a,s},F_a(sz))}
      {Z_{\tau s}(f_0,F_a(sz))}
 =\frac1{f_0}(\gamma_\tau*\psi)(z)
 +\rho_{a,s}(z,\tau),
\end{equation}
where
\begin{equation}
\label{eq:ratio-remainder}
 \sup_{\substack{a\in\cM,\ \|z\|\le R,\ \tau\in I}}
 \left\{|\rho_{a,s}(z,\tau)|
 +|\tau\partial_\tau\rho_{a,s}(z,\tau)|\right\}
 \le C_Rs^2.
\end{equation}
\end{lemma}

\begin{proof}
Choose $s_R$ so that $s_R(R\vee R_b)<r_0/4$.  With $x=F_a(sz)$ and the
change of variables $v=sw$, the numerator is
\begin{equation}
\label{eq:ratio-numerator-chart}
 s^d\int
 \exp\!\left(-\frac{\|F_a(sw)-F_a(sz)\|^2}{2\tau^2s^2}\right)
 \{\psi(w)-c_{a,s}\chi(w)\}J_a(sw)\,\dd w.
\end{equation}
For $\|z\|\le R$ and $w$ in the fixed support of the profiles, the graph
identity \eqref{eq:graph-distance-identity}, the bound
$\|DG_a(v)\|\le C\|v\|$, and
\eqref{eq:jacobian-quadratic-bound} imply
\begin{align}
 \frac{\|F_a(sw)-F_a(sz)\|^2}{s^2}
 &=\|w-z\|^2+O(s^2),
 \label{eq:local-distance-ratio}\\
 J_a(sw)&=1+O(s^2),
 \label{eq:local-jacobian-ratio}
\end{align}
where both remainders are uniform in $a,z,w$.  The bounds remain valid after
one $\tau\partial_\tau$ derivative of the exponential, because
$\tau\in I$.  Together with $c_{a,s}=O(s^2)$, the mean-value theorem in
\eqref{eq:ratio-numerator-chart} gives
\begin{align}
 Z_{\tau s}(b_{a,s},F_a(sz))
 =s^d(2\pi)^{d/2}\tau^d
 \{(\gamma_\tau*\psi)(z)+R^{(1)}_{a,s}(z,\tau)\},
\end{align}
where
\[
 |R^{(1)}_{a,s}|+|\tau\partial_\tau R^{(1)}_{a,s}|
 \le C_Rs^2.
\]

For the denominator, apply
Lemma~\ref{lem:differentiable-kernel-mass} to the constant density $f_0$ at
the point $F_a(sz)$ and the scale $\tau s$.  Uniformly for $\tau\in I$,
\begin{align*}
 Z_{\tau s}(f_0,F_a(sz))
 =s^d(2\pi)^{d/2}\tau^df_0
 \{1+R^{(0)}_{a,s}(z,\tau)\},
\end{align*}
with
\[
 |R^{(0)}_{a,s}|+|\tau\partial_\tau R^{(0)}_{a,s}|
 \le C_Rs^2.
\]
After decreasing $s_R$, the denominator factor in braces is at least
$1/2$.  Division and the quotient rule prove
\eqref{eq:ratio-expansion}--\eqref{eq:ratio-remainder}.
\end{proof}

For $r>0$, define
\begin{equation}
\label{eq:flipd-response-operator}
 \begin{aligned}
 (\mathcal A_r\psi)(z)
 &:=r\partial_r(\gamma_r*\psi)(z)
 =r^2\Delta(\gamma_r*\psi)(z),\\
 \mathcal A\psi&:=\mathcal A_1\psi.
 \end{aligned}
\end{equation}
The second equality follows from the Gaussian heat equation.

\begin{lemma}\label{lem:flipd-response-positive}
There exist $R>1$ and a nonzero mean-zero profile
$\psi_{\FLIPD}\in C_c^\infty(\R^d)$ such that
\begin{equation}
\label{eq:flipd-response-positive}
 E_A:=\int_{B_d(0,R)}
 |\mathcal A\psi_{\FLIPD}(z)|^2\,\dd z>0.
\end{equation}
\end{lemma}

\begin{proof}
Choose $\zeta\in C_c^\infty(\R^d)$ such that
$\partial_1\zeta\not\equiv0$, and set
$\psi_{\FLIPD}=\partial_1\zeta$.  Then
$\int\psi_{\FLIPD}=0$.  Under the Fourier transform, $\mathcal A_r$ has
multiplier
\[
 -r^2\|\xi\|^2e^{-r^2\|\xi\|^2/2}.
\]
If $\mathcal A\psi_{\FLIPD}\equiv0$, then
$\widehat\psi_{\FLIPD}(\xi)=0$ for every $\xi\ne0$.  At $\xi=0$ it also
vanishes because $\int\psi_{\FLIPD}=0$.  Injectivity of the Fourier
transform would give $\psi_{\FLIPD}=0$, a contradiction.  Hence
$\mathcal A\psi_{\FLIPD}$ is nonzero at some point.  By continuity, its
squared modulus has positive integral on a bounded ball; enlarging that ball
if necessary gives $R>1$.
\end{proof}

\begin{lemma}\label{lem:flipd-one-bump}
Let $\psi_{\FLIPD}$ be the profile from
Lemma~\ref{lem:flipd-response-positive}, and let $b_{a,s}$ be its manifold
bump.  There exist constants $s_1,c_1,C_q>0$ such that, for every
$a\in\cM$, $0<s\le s_1$, and $0<\theta\le s^2$, the functions
$f_\pm=f_0\pm\theta b_{a,s}$ are positive densities.  With
\begin{equation}
\label{eq:query-region-definition}
 U_{a,s}:=F_a(sB_d(0,R)),
\end{equation}
one has
\begin{align}
 U_{a,s}&\subset\cM\cap B_D(a,C_qs),
 \qquad \vol_\cM(U_{a,s})\le C_qs^d,
 \label{eq:query-region-localization}\\
 \int_{U_{a,s}}
 |T_s(x;f_+)-T_s(x;f_-)|^2\,\dd\vol_\cM(x)
 &\ge c_1\theta^2s^d.
 \label{eq:flipd-one-bump}
\end{align}
\end{lemma}

\begin{proof}
Choose $s_1$ so that $s_1R<r_0/4$ and
$s_1^2C_b\le f_0/2$.  Lemma~\ref{lem:scaled-bump} then shows that
$f_\pm\ge f_0/2$ and that both functions integrate to one.

For $x=F_a(sz)$ and $r\in I$, define
\[
 q_{r,s}(z):=
 \frac{Z_{rs}(b_{a,s},x)}{Z_{rs}(f_0,x)}.
\]
Lemma~\ref{lem:ratio-expansion} gives, uniformly for
$z\in B_d(0,R)$ and $r\in I$,
\begin{align}
 q_{r,s}(z)
 &=f_0^{-1}(\gamma_r*\psi_{\FLIPD})(z)+O(s^2),
 \label{eq:q-expansion}\\
 r\partial_rq_{r,s}(z)
 &=f_0^{-1}(\mathcal A_r\psi_{\FLIPD})(z)+O(s^2).
 \label{eq:q-derivative-expansion}
\end{align}
In particular, $|q_{r,s}(z)|\le C$.  Decrease $s_1$ so that
$\theta|q_{r,s}(z)|\le1/2$ whenever $\theta\le s^2$.  The logarithms below
are then well defined, and
\begin{align*}
 T_s(x;f_+)-T_s(x;f_-)
 &=\left.r\partial_r
 \log\frac{1+\theta q_{r,s}(z)}{1-\theta q_{r,s}(z)}
 \right|_{r=1}\\
 &=\left.
 \frac{2\theta r\partial_rq_{r,s}(z)}
 {1-\theta^2q_{r,s}(z)^2}
 \right|_{r=1}.
\end{align*}
Using \eqref{eq:q-expansion}--\eqref{eq:q-derivative-expansion} and
$(1-y)^{-1}=1+O(y)$ for $|y|\le1/4$, we obtain
\begin{equation}
\label{eq:one-bump-response-expansion}
 T_s(F_a(sz);f_+)-T_s(F_a(sz);f_-)
 =\frac{2\theta}{f_0}\mathcal A\psi_{\FLIPD}(z)
 +\rho_s(z),
\end{equation}
where
\begin{equation}
\label{eq:one-bump-response-remainder}
 \sup_{a\in\cM}\sup_{z\in B_d(0,R)}|\rho_s(z)|
 \le C(\theta s^2+\theta^3).
\end{equation}
Since $\theta\le s^2\le1$, the right-hand side is at most $C\theta s^2$.

The bi-Lipschitz and Jacobian bounds give
\eqref{eq:query-region-localization}.  Changing variables $x=F_a(sz)$ in
the squared response, using $J_a(sz)\ge c_0$, and applying
$|u+v|^2\ge |u|^2/2-|v|^2$ yield
\begin{align*}
 &\int_{U_{a,s}}|T_s(x;f_+)-T_s(x;f_-)|^2
 \,\dd\vol_\cM(x)\\
 &\qquad\ge
 cs^d\left\{
 \frac{2\theta^2}{f_0^2}E_A-C\theta^2s^4
 \right\}.
\end{align*}
After decreasing $s_1$ once more, the term in braces is bounded below by
$c\theta^2$.  This proves \eqref{eq:flipd-one-bump} uniformly in $a$.
\end{proof}

The last auxiliary lemma upgrades the preceding isolated separation to the
full Assouad hypercube.

\begin{lemma}\label{lem:multi-bump-stability}
There exist a fixed $L_0>0$ and constants $s_2,c_2>0$ such that the
following holds.  Let $0<s\le s_2$, let
$a_1,\ldots,a_m\in K$ be ambiently $L_0s$-separated, and let
$b_k=b_{a_k,s}$.  For every $\omega\in\{-1,1\}^m$, every $j$, and every
$0<\theta\le s^2$,
\begin{equation}
\label{eq:multi-bump-response}
 \int_{U_{a_j,s}}
 \left|T_s(x;f_\omega)-T_s(x;f_{\omega^{(j)}})\right|^2
 \,\dd\vol_\cM(x)
 \ge c_2\theta^2s^d,
\end{equation}
where
$f_\omega=f_0+\theta\sum_{k=1}^m\omega_kb_k$.
\end{lemma}

\begin{proof}
Choose $L_0$ large enough that Lemma~\ref{lem:shell-decoupling} applies and
$L_0>2(C_b+C_q)$.  Fix the signs $\omega_k$, $k\ne j$, and define, for
$|t|\le\theta$,
\begin{equation}
\label{eq:density-paths}
 g_t:=f_0+\theta\sum_{k\ne j}\omega_kb_k+t b_j,
 \qquad
 g_t^{(0)}:=f_0+t b_j.
\end{equation}
The bump supports are disjoint.  Hence, after decreasing $s_2$ so that
$s_2^2C_b\le f_0/2$, both paths consist of positive densities satisfying
\begin{equation}
\label{eq:path-density-bounds}
 f_0/2\le g_t,g_t^{(0)}\le 3f_0/2
 \qquad (|t|\le\theta).
\end{equation}

For $r\in I$ and $x\in U_{a_j,s}$, put
\[
 D_r(t):=Z_{rs}(g_t,x),
 \qquad
 D_r^{(0)}(t):=Z_{rs}(g_t^{(0)},x),
 \qquad
 B_r:=Z_{rs}(b_j,x).
\]
The denominator bounds use only the pointwise inequalities in
\eqref{eq:path-density-bounds}; the path densities need not themselves be
members of $\mathcal F_\alpha(\cM)$.  Indeed,
\begin{equation}
\label{eq:path-denominator-comparison}
 \frac{f_0}{2}\int_\cM K_{rs,x}(u)\,\dd\vol_\cM(u)
 \le D_r(t),D_r^{(0)}(t)
 \le\frac{3f_0}{2}\int_\cM K_{rs,x}(u)\,\dd\vol_\cM(u).
\end{equation}
The chart and complement calculation in the proof of
Lemma~\ref{lem:localization}, applied to the constant integrand, shows that
the kernel integral in \eqref{eq:path-denominator-comparison} lies between
$cs^d$ and $Cs^d$, uniformly for $r\in I$.  Furthermore,
\begin{equation}
\label{eq:path-kernel-scale-derivative}
 |r\partial_rK_{rs,x}(u)|
 =\frac{\|u-x\|^2}{r^2s^2}K_{rs,x}(u)
 \le C\left(1+\frac{\|u-x\|}{s}\right)^2
 e^{-c\|u-x\|^2/s^2}.
\end{equation}
The same chart calculation with the polynomial weight in
\eqref{eq:path-kernel-scale-derivative} bounds its integral by $Cs^d$.
Finally, $\|b_j\|_\infty\le C$, its support has volume at most $Cs^d$,
and, for $x\in U_{a_j,s}$ and $u\in\supp(b_j)$,
$\|u-x\|\le(C_b+C_q)s$.  These observations give, uniformly in
$r\in I$, $|t|\le\theta$, and $x\in U_{a_j,s}$,
\begin{align}
 cs^d\le D_r(t),D_r^{(0)}(t)&\le Cs^d,
 \label{eq:path-denominator-bounds}\\
 |B_r|+|r\partial_rB_r|
 +|r\partial_rD_r(t)|+|r\partial_rD_r^{(0)}(t)|&\le Cs^d.
 \label{eq:path-derivative-bounds}
\end{align}
Define the remote contribution
\[
 E_r:=D_r(t)-D_r^{(0)}(t)
 =\theta\sum_{k\ne j}\omega_kZ_{rs}(b_k,x).
\]
It is independent of $t$.  Lemma~\ref{lem:shell-decoupling}, with zero and
one kernel-scale derivative, yields
\begin{equation}
\label{eq:remote-denominator-bound}
 |E_r|+|r\partial_rE_r|
 \le C\theta s^d\delta(L_0).
\end{equation}

Set
\[
 H_x(t):=\left.r\partial_r\log D_r(t)\right|_{r=1},
 \qquad
 H_x^{(0)}(t):=\left.r\partial_r\log D_r^{(0)}(t)\right|_{r=1}.
\]
All derivatives below may be passed under the integral by dominated
convergence; for $r\in I$ the differentiated kernels are bounded by a fixed
polynomial times a Gaussian, and the denominators are bounded away from zero
by \eqref{eq:path-denominator-bounds}.  Since
$\partial_tD_r(t)=B_r$, direct differentiation gives
\begin{equation}
\label{eq:flipd-path-derivatives}
 \begin{aligned}
 H_x'(0)
 &=\left.r\partial_r\frac{B_r}{D_r(0)}\right|_{r=1},\\
 H_x'''(t)
 &=\left.r\partial_r
 \left(2\frac{B_r^3}{D_r(t)^3}\right)\right|_{r=1}.
 \end{aligned}
\end{equation}
For clarity, write a dot for $r\partial_r$.  Then
\[
 \left(\frac{B_r}{D_r}\right)^{\boldsymbol\cdot}
 =\frac{\dot B_r}{D_r}-\frac{B_r\dot D_r}{D_r^2}.
\]
Using \eqref{eq:path-denominator-bounds}--
\eqref{eq:remote-denominator-bound} in this identity, once with $D_r$ and
once with $D_r^{(0)}$, gives
\begin{equation}
\label{eq:first-path-derivative-stability}
 |H_x'(0)-(H_x^{(0)})'(0)|
 \le C\theta\delta(L_0).
\end{equation}
Similarly,
\[
 \left(2\frac{B_r^3}{D_r(t)^3}\right)^{\boldsymbol\cdot}
 =6\frac{B_r^2\dot B_r}{D_r(t)^3}
 -6\frac{B_r^3\dot D_r(t)}{D_r(t)^4},
\]
so \eqref{eq:path-denominator-bounds}--
\eqref{eq:path-derivative-bounds} imply
\begin{equation}
\label{eq:third-path-derivative-bound}
 \sup_{|t|\le\theta}
 \{|H_x'''(t)|+|(H_x^{(0)})'''(t)|\}
 \le C.
\end{equation}

Taylor's theorem at $t=0$ gives
\[
 |H_x(\theta)-H_x(-\theta)-2\theta H_x'(0)|
 \le\frac{\theta^3}{3}\sup_{|t|\le\theta}|H_x'''(t)|,
\]
and the same bound holds for $H_x^{(0)}$.  Combining this with
\eqref{eq:first-path-derivative-stability} and
\eqref{eq:third-path-derivative-bound} yields, uniformly on
$U_{a_j,s}$,
\begin{equation}
\label{eq:pointwise-remote-response-bound}
 \left|
 \{H_x(\theta)-H_x(-\theta)\}
 -\{H_x^{(0)}(\theta)-H_x^{(0)}(-\theta)\}
 \right|
 \le C\{\theta^2\delta(L_0)+\theta^3\}.
\end{equation}
By \eqref{eq:query-region-localization}, the $L^2(U_{a_j,s})$ norm of
the difference in \eqref{eq:pointwise-remote-response-bound} is at most
\[
 C\{\theta^2\delta(L_0)+\theta^3\}s^{d/2}.
\]
On the other hand, Lemma~\ref{lem:flipd-one-bump} gives
\[
 \|H_\cdot^{(0)}(\theta)-H_\cdot^{(0)}(-\theta)\|_{L^2(U_{a_j,s})}
 \ge \sqrt{c_1}\,\theta s^{d/2}.
\]
Since $\theta\le s^2$, the ratio of the former norm to the latter is at
most
\[
 C\{s^2\delta(L_0)+s^4\}.
\]
After $L_0$ is fixed, decrease $s_2$ so that this ratio is at most $1/2$.
The triangle inequality then gives
\[
 \|H_\cdot(\theta)-H_\cdot(-\theta)\|_{L^2(U_{a_j,s})}
 \ge \frac{\sqrt{c_1}}2\theta s^{d/2}.
\]
Squaring proves \eqref{eq:multi-bump-response} with
$c_2=c_1/4$. 
\end{proof}

\subsection{Main Proof of Theorem~\ref{thm:minimax-lower}}

\begin{proof}
Let $L_0$ be the fixed separation constant from
Lemma~\ref{lem:multi-bump-stability}.  That lemma verifies the adjacent-edge
condition \eqref{eq:generic-one-bump} in
Proposition~\ref{prop:assouad-reduction} with $\kappa=0$, the query regions
\eqref{eq:query-region-definition}, and the constant $c_A=c_2$.

Let $\sigma_0>0$ be the minimum of $1$ and the finitely many upper-scale
constants used in Lemmas~\ref{lem:scaled-bump},
\ref{lem:shell-decoupling}, \ref{lem:ratio-expansion},
\ref{lem:flipd-one-bump}, and \ref{lem:multi-bump-stability}, and in
Proposition~\ref{prop:assouad-reduction}.  Choose $n_0$ so that
\[
 h_n=n^{-1/(2\alpha+d)}\le\sigma_0
 \qquad\text{for every }n\ge n_0.
\]
For $n\ge n_0$ and $h_n\le\sigma\le\sigma_0$, apply
Proposition~\ref{prop:assouad-reduction} with $s=\sigma$ and $\kappa=0$.
It follows that
\[
 \mathfrak R_{n,\sigma}
 \ge c(n\sigma^d)^{-1},
\]
where $c>0$ is independent of $n$ and $\sigma$.  This is
\eqref{eq:main-flipd-lower} and proves
Theorem~\ref{thm:minimax-lower}.
\end{proof}

\section{Empirical DSM and Proof of Theorem~\ref{thm:minimax-upper}}
\label{app:minimax-upper}

The argument has two parts. First, completion of squares in the empirical
DSM objective and direct differentiation of Gaussian kernel sums identify
the minimizer and its weighted-moment representation. Second, a
deterministic clipped-ratio inequality reduces
the integrated risk to the second moments of two unbiased empirical
averages.

\subsection{Proof of Proposition~\ref{prop:empirical-dsm}}

\begin{proof}
Fix $X_1,\ldots,X_n$ and $\sigma>0$, and write
\[
 \varphi_{\sigma,i}(y)
 :=(2\pi)^{-D/2}\sigma^{-D}K_{\sigma,y}(X_i),
 \qquad
 \widehat p_\sigma(y)=\frac1n\sum_{i=1}^n\varphi_{\sigma,i}(y).
\]
Under the change of variables $y=X_i+\sigma\varepsilon$, the standard
Gaussian measure in $\varepsilon$ becomes
$\varphi_{\sigma,i}(y)\,\dd y$, and
$\varepsilon/\sigma=(y-X_i)/\sigma^2$. Thus
\eqref{eq:empirical-dsm-objective} gives
\begin{equation}
\label{eq:empirical-dsm-integral}
 \widehat{\mathcal L}_\sigma(s)
 =\frac1n\sum_{i=1}^n\int_{\R^D}
   \left\|s(y)+\frac{y-X_i}{\sigma^2}\right\|^2
   \varphi_{\sigma,i}(y)\,\dd y.
\end{equation}
Since each Gaussian density is strictly positive and smooth,
$\widehat p_\sigma$ has these properties as well. Moreover,
\[
 \nabla_y\varphi_{\sigma,i}(y)
 =\frac{X_i-y}{\sigma^2}\varphi_{\sigma,i}(y),
 \qquad
 \frac{\varphi_{\sigma,i}(y)}{n\widehat p_\sigma(y)}
 =w_{i,\sigma}(y).
\]
Consequently,
\begin{equation}
\label{eq:empirical-dsm-weighted-score}
 \begin{aligned}
 s_\sigma^*(y):=\nabla\log\widehat p_\sigma(y)
 &=\frac{\frac1n\sum_{i=1}^n
       \varphi_{\sigma,i}(y)(X_i-y)/\sigma^2}
       {\widehat p_\sigma(y)}\\
 &=\sum_{i=1}^nw_{i,\sigma}(y)\frac{X_i-y}{\sigma^2}
 =\frac{\widehat m_\sigma(y)-y}{\sigma^2}.
 \end{aligned}
\end{equation}
In particular, $s_\sigma^*$ is smooth. Since the weights are nonnegative
and sum to one, convexity of the squared norm gives
\[
 \|s_\sigma^*(y)\|^2
 \le\sum_{i=1}^nw_{i,\sigma}(y)
       \left\|\frac{X_i-y}{\sigma^2}\right\|^2.
\]
Using $\widehat p_\sigma w_{i,\sigma}=\varphi_{\sigma,i}/n$ and then
$y=X_i+\sigma\varepsilon$, we obtain
\begin{equation}
\label{eq:empirical-score-square-integrability}
 \begin{aligned}
 \int_{\R^D}\|s_\sigma^*(y)\|^2\widehat p_\sigma(y)\,\dd y
 &\le\frac1n\sum_{i=1}^n\int_{\R^D}
       \left\|\frac{X_i-y}{\sigma^2}\right\|^2
       \varphi_{\sigma,i}(y)\,\dd y\\
 &=\frac1{n\sigma^2}\sum_{i=1}^n
       \E_{\varepsilon\sim N(0,I_D)}\|\varepsilon\|^2
 =\frac{D}{\sigma^2}.
 \end{aligned}
\end{equation}
Hence $s_\sigma^*\in\mathcal S_\sigma$.

For every $s\in\mathcal S_\sigma$, expand the square in
\eqref{eq:empirical-dsm-integral} using
\[
 s(y)+\frac{y-X_i}{\sigma^2}
 =\{s(y)-s_\sigma^*(y)\}
  +\left\{s_\sigma^*(y)+\frac{y-X_i}{\sigma^2}\right\}.
\]
The cross term vanishes because, pointwise in $y$,
\begin{align*}
 &\frac1n\sum_{i=1}^n
   \left\{s_\sigma^*(y)+\frac{y-X_i}{\sigma^2}\right\}
   \varphi_{\sigma,i}(y)\\
 &\qquad=\widehat p_\sigma(y)
   \left\{s_\sigma^*(y)
       +\sum_{i=1}^nw_{i,\sigma}(y)\frac{y-X_i}{\sigma^2}\right\}
 =0
\end{align*}
by \eqref{eq:empirical-dsm-weighted-score}. More explicitly, expanding
and integrating yields
\begin{align*}
 \widehat{\mathcal L}_\sigma(s)
 &=\widehat{\mathcal L}_\sigma(s_\sigma^*)
   +\int_{\R^D}\|s(y)-s_\sigma^*(y)\|^2
                   \widehat p_\sigma(y)\,\dd y\\
 &\quad+\frac2n\sum_{i=1}^n\int_{\R^D}
   \left\langle s(y)-s_\sigma^*(y),
      s_\sigma^*(y)+\frac{y-X_i}{\sigma^2}\right\rangle
      \varphi_{\sigma,i}(y)\,\dd y,
\end{align*}
where the last line is zero. All these terms are integrable by the
definition of $\mathcal S_\sigma$,
\eqref{eq:empirical-score-square-integrability}, and the Gaussian second
moments. Consequently,
\begin{equation}
\label{eq:empirical-dsm-pythagorean}
 \widehat{\mathcal L}_\sigma(s)
 -\widehat{\mathcal L}_\sigma(s_\sigma^*)
 =\int_{\R^D}\|s(y)-s_\sigma^*(y)\|^2
                  \widehat p_\sigma(y)\,\dd y
 \ge0.
\end{equation}
Equality implies $s=s_\sigma^*$ almost everywhere. Since both functions
are continuous and $\widehat p_\sigma>0$ everywhere, they must in fact
agree everywhere: a nonzero difference at one point would give a
strictly positive integral over a neighborhood of that point. This
proves \eqref{eq:empirical-dsm-minimizer}.

Next, differentiating the Gaussian weights gives
\begin{align*}
 \nabla_x\log w_{i,\sigma}(x)
 &=\nabla_x\log K_{\sigma,x}(X_i)
   -\nabla_x\log\left\{\sum_{j=1}^nK_{\sigma,x}(X_j)\right\}\\
 &=\frac{X_i-x}{\sigma^2}
   -\sum_{j=1}^nw_{j,\sigma}(x)\frac{X_j-x}{\sigma^2}\\
 &=\frac{X_i-\widehat m_\sigma(x)}{\sigma^2}.
\end{align*}
Multiplication by $w_{i,\sigma}(x)$ therefore yields
\[
 \nabla_xw_{i,\sigma}(x)
 =\frac{w_{i,\sigma}(x)}{\sigma^2}
       \{X_i-\widehat m_\sigma(x)\}.
\]
Since
\[
 \sum_{i=1}^nw_{i,\sigma}(x)
       \{X_i-\widehat m_\sigma(x)\}=0,
\]
the Jacobian of the weighted mean satisfies
\begin{equation}
\label{eq:empirical-mean-jacobian}
 \begin{aligned}
 \nabla_x\widehat m_\sigma(x)
 &=\sum_{i=1}^nX_i\{\nabla_xw_{i,\sigma}(x)\}^{\mathsf T}\\
 &=\frac1{\sigma^2}\sum_{i=1}^nw_{i,\sigma}(x)
       X_i\{X_i-\widehat m_\sigma(x)\}^{\mathsf T}\\
 &=\frac1{\sigma^2}\sum_{i=1}^nw_{i,\sigma}(x)
       \{X_i-\widehat m_\sigma(x)\}
       \{X_i-\widehat m_\sigma(x)\}^{\mathsf T}
 =\frac{\widehat C_\sigma(x)}{\sigma^2}.
 \end{aligned}
\end{equation}
Together with \eqref{eq:empirical-dsm-weighted-score}, this gives
\[
 \nabla_x\widehat s_\sigma(x)
 =\frac{\nabla_x\widehat m_\sigma(x)-I_D}{\sigma^2}
 =-\frac{I_D}{\sigma^2}
   +\frac{\widehat C_\sigma(x)}{\sigma^4},
\]
which proves \eqref{eq:empirical-score-covariance}. Taking the trace and
using the score representation, we obtain
\begin{align*}
 &D+\sigma^2\left\{
       \operatorname{div}\widehat s_\sigma(x)
       +\|\widehat s_\sigma(x)\|^2\right\}\\
 &\qquad=D+\sigma^2\left\{
       -\frac D{\sigma^2}
       +\frac{\tr\widehat C_\sigma(x)}{\sigma^4}
       +\frac{\|\widehat m_\sigma(x)-x\|^2}{\sigma^4}
       \right\}\\
 &\qquad=\frac{\tr\widehat C_\sigma(x)
             +\|\widehat m_\sigma(x)-x\|^2}{\sigma^2}.
\end{align*}
Finally, expanding
$X_i-x=\{X_i-\widehat m_\sigma(x)\}+\{\widehat m_\sigma(x)-x\}$ gives
\begin{align*}
 \sum_{i=1}^nw_{i,\sigma}(x)\|X_i-x\|^2
 &=\sum_{i=1}^nw_{i,\sigma}(x)
       \|X_i-\widehat m_\sigma(x)\|^2
   +\|\widehat m_\sigma(x)-x\|^2\\
 &\quad+2\left\langle
       \sum_{i=1}^nw_{i,\sigma}(x)
          \{X_i-\widehat m_\sigma(x)\},
       \widehat m_\sigma(x)-x\right\rangle\\
 &=\tr\widehat C_\sigma(x)
   +\|\widehat m_\sigma(x)-x\|^2.
\end{align*}
Substituting the definition of $w_{i,\sigma}(x)$ therefore gives
\[
 \widehat T_\sigma^{\mathrm{raw}}(x)
 =\frac1{\sigma^2}\sum_{i=1}^nw_{i,\sigma}(x)\|X_i-x\|^2
 =\frac{\sum_{i=1}^n\|X_i-x\|^2K_{\sigma,x}(X_i)}
        {\sigma^2\sum_{i=1}^nK_{\sigma,x}(X_i)}.
\]
These identities prove \eqref{eq:empirical-dsm-flipd}.
\end{proof}

\subsection{A Clipped-Ratio Inequality}

The next inequality controls a clipped empirical ratio using the
population denominator. It avoids a separate lower-tail event for the
empirical kernel mass, and does not require independence of the empirical
numerator and denominator.

\begin{lemma}
\label{lem:clipped-ratio}
Let $z,\widehat z>0$ and $a,\widehat a\in\R$, with
$t:=a/z\in[0,D+1]$. If
$\widehat t:=0\lor(\widehat a/\widehat z)\land(D+1)$, then
\begin{equation}
\label{eq:clipped-ratio-inequality}
 |\widehat t-t|
 \le\frac{|\widehat a-a|+(D+1)|\widehat z-z|}{z}.
\end{equation}
\end{lemma}

\begin{proof}
Write $u:=\widehat a/\widehat z$. Since $\widehat t$ is the Euclidean
projection of $u$ onto $[0,D+1]$ and $t$ belongs to this interval,
\[
 (\widehat t-u)(t-\widehat t)\ge0.
\]
Multiplying by $\widehat z>0$ and reversing the sign of the second
factor gives
\[
 (\widehat z\widehat t-\widehat a)(\widehat t-t)\le0.
\]
Writing $e:=\widehat t-t$ and using $a=zt$, we have
\begin{align*}
 ze^2
 &=e(z\widehat t-a)\\
 &=e(\widehat z\widehat t-\widehat a)
   +e\{(z-\widehat z)\widehat t+(\widehat a-a)\}\\
 &\le e\{(z-\widehat z)\widehat t+(\widehat a-a)\}\\
 &\le |e|\left\{|z-\widehat z|\,|\widehat t|
                     +|\widehat a-a|\right\}\\
 &\le |e|\{(D+1)|z-\widehat z|+|\widehat a-a|\},
\end{align*}
where the last step uses $0\le\widehat t\le D+1$.
For $e\ne0$, divide by $z|e|$ to obtain
\eqref{eq:clipped-ratio-inequality}; the case $e=0$ is immediate.
\end{proof}

\subsection{Main Proof of Theorem~\ref{thm:minimax-upper}}

\begin{proof}
For $f\in\mathcal F_\alpha(\cM)$, $x\in\cM$, and $\sigma>0$, define
\begin{align}
 q_{\sigma,x}(u)
 &:=\frac{\|u-x\|^2}{\sigma^2},
 \qquad
 N_\sigma(f,x)
 :=\int_\cM q_{\sigma,x}(u)K_{\sigma,x}(u)\,\dd P_f(u),
 \notag\\
 \widehat N_\sigma(x)
 &:=\frac1n\sum_{i=1}^n
         q_{\sigma,x}(X_i)K_{\sigma,x}(X_i).
 \label{eq:flipd-moment-numerators}
\end{align}
Direct differentiation gives
\[
 \sigma\partial_\sigma K_{\sigma,x}(u)
 =\sigma K_{\sigma,x}(u)\frac{\|u-x\|^2}{\sigma^3}
 =q_{\sigma,x}(u)K_{\sigma,x}(u).
\]
For each fixed $\sigma>0$, differentiation under the integral is
justified by compactness of $\cM$ and smoothness of the kernel. Thus
\[
 \sigma\partial_\sigma Z_\sigma(f,x)=N_\sigma(f,x),
 \qquad
 T_\sigma(x;f)
 =\sigma\partial_\sigma\log Z_\sigma(f,x)
 =\frac{\sigma\partial_\sigma Z_\sigma(f,x)}{Z_\sigma(f,x)}.
\]
Combining this identity with the empirical representation already proved
in Proposition~\ref{prop:empirical-dsm}, we obtain
\begin{equation}
\label{eq:population-empirical-flipd-ratios}
 T_\sigma(x;f)=\frac{N_\sigma(f,x)}{Z_\sigma(f,x)},
 \qquad
 \widehat T_\sigma^{\mathrm{raw}}(x)
 =\frac{\widehat N_\sigma(x)}{\widehat Z_\sigma(x)}.
\end{equation}
The denominators are strictly positive because the Gaussian kernel is
strictly positive and $n\ge1$. The population numerator is nonnegative,
so $T_\sigma(x;f)\ge0$.
Let $C_{\mathrm{bias}}$ denote a uniform constant in
Theorem~\ref{thm:small-noise-limits}. Decrease $\sigma_0$ so that
$C_{\mathrm{bias}}\sigma_0^2\le1$. Then
\begin{equation}
\label{eq:flipd-clipping-target-range}
 0\le T_\sigma(x;f)
 \le d+C_{\mathrm{bias}}\sigma^2
 \le d+1\le D<D+1,
 \qquad 0<\sigma\le\sigma_0,
\end{equation}
uniformly over $f$ and $x$.

We next control the population denominator and the second moments of
the two empirical summands. The mass bound in
Lemma~\ref{lem:localization} gives $Z_\sigma(f,x)\ge c\sigma^d$.
Its weighted kernel bound, applied with $p=2$ and $m=0$, gives
\[
 \E_f[K_{\sigma,x}(X_1)^2]
 =\int_\cM K_{\sigma,x}(u)^2\,\dd P_f(u)
 \le C\sigma^d.
\]
For the numerator, the same bound with $p=2$ and $m=4$ gives
\begin{align*}
 \E_f[q_{\sigma,x}(X_1)^2K_{\sigma,x}(X_1)^2]
 &=\int_\cM
       \left(\frac{\|u-x\|}{\sigma}\right)^4
       K_{\sigma,x}(u)^2\,\dd P_f(u)\\
 &\le\int_\cM
       \left(1+\frac{\|u-x\|}{\sigma}\right)^4
       K_{\sigma,x}(u)^2\,\dd P_f(u)
 \le C\sigma^d.
\end{align*}
Consequently,
\begin{align}
 Z_\sigma(f,x)&\ge c\sigma^d,
 \notag\\
 \E_f[K_{\sigma,x}(X_1)^2]
 +\E_f[q_{\sigma,x}(X_1)^2K_{\sigma,x}(X_1)^2]
 &\le C\sigma^d,
 \label{eq:upper-localized-second-moments}
\end{align}
uniformly over $f$, $x$, and $0<\sigma\le\sigma_0$, after a further
class-dependent decrease of $\sigma_0$ if needed.

By the definitions of the population moments,
\[
 \E_f\widehat Z_\sigma(x)=Z_\sigma(f,x),
 \qquad
 \E_f\widehat N_\sigma(x)=N_\sigma(f,x).
\]
For the denominator, independence across observations yields
\begin{align*}
 &\E_f|\widehat Z_\sigma(x)-Z_\sigma(f,x)|^2\\
 &\qquad=\frac1{n^2}\sum_{i,j=1}^n
   \E_f\left[
      \{K_{\sigma,x}(X_i)-Z_\sigma(f,x)\}
      \{K_{\sigma,x}(X_j)-Z_\sigma(f,x)\}
   \right]\\
 &\qquad=\frac1{n^2}\sum_{i=1}^n
   \E_f|K_{\sigma,x}(X_i)-Z_\sigma(f,x)|^2
 =\frac1n\operatorname{Var}_f(K_{\sigma,x}(X_1)).
\end{align*}
Here the terms with $i\ne j$ vanish because the centered summands are
independent and have mean zero. The same calculation applies to
$\widehat N_\sigma$. Using
$\operatorname{Var}_f(Y)\le\E_f[Y^2]$ and
\eqref{eq:upper-localized-second-moments}, we therefore obtain
\begin{align}
 \E_f|\widehat Z_\sigma(x)-Z_\sigma(f,x)|^2
 &=\frac1n\operatorname{Var}_f(K_{\sigma,x}(X_1))
 \le\frac{C\sigma^d}{n},
 \notag\\
 \E_f|\widehat N_\sigma(x)-N_\sigma(f,x)|^2
 &=\frac1n\operatorname{Var}_f\!\left(
       q_{\sigma,x}(X_1)K_{\sigma,x}(X_1)\right)
 \le\frac{C\sigma^d}{n}.
 \label{eq:upper-empirical-moment-errors}
\end{align}

Apply Lemma~\ref{lem:clipped-ratio} with
\[
 a=N_\sigma(f,x),\quad z=Z_\sigma(f,x),\quad
 \widehat a=\widehat N_\sigma(x),\quad
 \widehat z=\widehat Z_\sigma(x).
\]
Its assumptions hold by
\eqref{eq:population-empirical-flipd-ratios} and
\eqref{eq:flipd-clipping-target-range}. Hence, for every realization of
the observations,
\[
 |\widehat T_\sigma(x)-T_\sigma(x;f)|
 \le\frac{
    |\widehat N_\sigma(x)-N_\sigma(f,x)|
    +(D+1)|\widehat Z_\sigma(x)-Z_\sigma(f,x)|}
    {Z_\sigma(f,x)}.
\]
Squaring this inequality, using $(u+v)^2\le2u^2+2v^2$, and taking
expectations gives
\begin{align}
 &\E_f|\widehat T_\sigma(x)-T_\sigma(x;f)|^2
 \notag\\
 &\quad\le\frac{2}{Z_\sigma(f,x)^2}
 \left\{
 \E_f|\widehat N_\sigma(x)-N_\sigma(f,x)|^2
 +(D+1)^2\E_f|\widehat Z_\sigma(x)-Z_\sigma(f,x)|^2
 \right\}
 \notag\\
 &\quad\le\frac{C}{\sigma^{2d}}\frac{\sigma^d}{n}
 =\frac{C}{n\sigma^d},
 \label{eq:upper-pointwise-risk-proof}
\end{align}
where the second inequality uses
\eqref{eq:upper-localized-second-moments} and
\eqref{eq:upper-empirical-moment-errors}. Only independence across
observations was used; the empirical numerator and denominator may be
dependent.

Both $\widehat T_\sigma(x)$ and $T_\sigma(x;f)$ lie in $[0,D+1]$, so
\[
 |\widehat T_\sigma(x)-T_\sigma(x;f)|^2\le(D+1)^2.
\]
Combining this bound with \eqref{eq:upper-pointwise-risk-proof} and
enlarging $C$ if necessary yields
\[
 \E_f|\widehat T_\sigma(x)-T_\sigma(x;f)|^2
 \le\min\left\{(D+1)^2,\frac{C}{n\sigma^d}\right\}
 \le C\left\{1\wedge(n\sigma^d)^{-1}\right\}.
\]
Since this bound is uniform in $x$ and $P_f$ is a probability measure,
Tonelli's theorem gives
\begin{align*}
 \mathcal R_{n,\sigma}(\widehat T_\sigma,f)
 &=\E_f\left[
      \int_\cM|\widehat T_\sigma(x)-T_\sigma(x;f)|^2\,\dd P_f(x)
      \right]\\
 &=\int_\cM
      \E_f|\widehat T_\sigma(x)-T_\sigma(x;f)|^2\,\dd P_f(x)\\
 &\le C\left\{1\wedge(n\sigma^d)^{-1}\right\}
       \int_\cM\dd P_f(x)
 =C\left\{1\wedge(n\sigma^d)^{-1}\right\}.
\end{align*}
Taking the supremum over $P_f\in\cP_\alpha(\cM,d)$ proves
\eqref{eq:main-flipd-upper}.

Finally, decrease $\sigma_0$ to the minimum of the upper-scale constants
in this proof and Theorem~\ref{thm:minimax-lower}. Increase $n_0$ if
necessary so that the lower-bound theorem applies and $h_n\le\sigma_0$
for every $n\ge n_0$. The definition of the minimax risk gives
\[
 \mathfrak R_{n,\sigma}
 =\inf_{\widetilde T_\sigma}
    \sup_{P_f\in\cP_\alpha(\cM,d)}
       \mathcal R_{n,\sigma}(\widetilde T_\sigma,f)
 \le\sup_{P_f\in\cP_\alpha(\cM,d)}
       \mathcal R_{n,\sigma}(\widehat T_\sigma,f).
\]
Thus, for $n\ge n_0$ and $h_n\le\sigma\le\sigma_0$, the lower bound and
the upper bound just proved imply
\[
 (n\sigma^d)^{-1}
 \lesssim\mathfrak R_{n,\sigma}
 \le\sup_{P_f\in\cP_\alpha(\cM,d)}
       \mathcal R_{n,\sigma}(\widehat T_\sigma,f)
 \lesssim(n\sigma^d)^{-1},
\]
which proves \eqref{eq:matching-flipd-minimax}.
\end{proof}

\end{document}